\documentclass[twoside]{article}

\usepackage[accepted]{aistats2026}
\makeatother  
\usepackage{babel} 
\usepackage{times}
\usepackage{helvet}
\usepackage{courier}
\usepackage{multicol}
\usepackage{soul}
\usepackage{xparse}
\usepackage{amsmath}
\usepackage{amsthm}
\usepackage{amssymb}
\usepackage{xspace}
\usepackage{relsize}
\usepackage{graphicx}
\usepackage{multirow}
\usepackage{comment}
\usepackage{mathtools}
\usepackage{dsfont}
\usepackage{enumerate}
\usepackage{enumitem}
\usepackage{mathabx}
\usepackage{mathrsfs}
\usepackage{footnote}
\usepackage{multirow}
\usepackage{xcolor}
\usepackage{hyperref}

\usepackage[round]{natbib}
\renewcommand{\cite}{\citep}

\newcommand{\EE}{\mathbb{E}}
\newcommand{\PP}{\mathrm{Pr}}
\newcommand{\Var}{\mathrm{Var}}

\theoremstyle{definition}
\newtheorem{definition}{Definition}
\newtheorem{theorem}[definition]{Theorem}
\newtheorem{lemma}[definition]{Lemma}
\newtheorem{corollary}[definition]{Corollary}
\newtheorem{remark}[definition]{Remark}

\def\eps{{\epsilon}}

\DeclareMathOperator*{\argmin}{argmin}

\begin{document}

%

%

\twocolumn[

    \aistatstitle{Sequential 1-bit Mean Estimation with Near-Optimal Sample Complexity}
    
    \aistatsauthor{ Ivan Lau \And Jonathan Scarlett }
    \aistatsaddress{ National University of Singapore \And National University of Singapore} 
]

\begin{abstract}
In this paper, we study the problem of distributed mean estimation with 1-bit communication constraints. We propose a mean estimator that is based on (randomized and sequentially-chosen) interval queries, whose 1-bit outcome indicates whether the given sample lies in the specified interval.
Our estimator is $(\epsilon, \delta)$-PAC for all distributions with bounded mean ($-\lambda \le \mathbb{E}(X) \le \lambda $) and variance ($\mathrm{Var}(X) \le \sigma^2$) for some known parameters $\lambda$ and $\sigma$. 
We derive a sample complexity bound $\widetilde{O}\big( \frac{\sigma^2}{\epsilon^2}\log\frac{1}{\delta} + \log\frac{\lambda}{\sigma}\big)$, 
which matches the minimax lower bound for the unquantized setting up to logarithmic factors and the additional $\log\frac{\lambda}{\sigma}$ term that we show to be unavoidable.
We also establish an adaptivity gap for interval-query based estimators: the best non-adaptive mean estimator is considerably worse than our adaptive mean estimator for large $\frac{\lambda}{\sigma}$. Finally, we give tightened sample complexity bounds for distributions with stronger tail decay, and present additional variants that (i) handle an unknown sampling budget (ii) adapt to the unknown true variance given (possibly loose) upper and lower bounds on the variance, and (iii) use only two stages of adaptivity at the expense of more complicated (non-interval) queries.
\end{abstract}
\section{INTRODUCTION}

Mean estimation is one of the simplest yet most ubiquitous tasks in statistics, machine learning, and theoretical computer science. In modern applications such as those arising in large-scale and decentralized systems, the learner often has limited access to the true data samples. A common limitation is communication constraints, which require each data sample to be compressed to a small number of bits, before being communicated to the learner. In this paper, we address the extreme case of this setting where the learner receives only one bit of feedback per sample. This raises a fundamental theoretical question: 
\begin{quote}
    \textit{How does 1-bit quantization affect the sample complexity of mean estimation?}
\end{quote}
Our main contribution is a 1-bit mean estimator whose sample complexity nearly matches the minimax lower bound for the unquantized setting. To the best of our knowledge, analogous results were previously established only for Gaussian random variables~\cite{cai2022distributed, cai2024distributed} and other known location-scale families such as Laplace and logistic \cite{kipnis2022mean, kumar2025one}, leaving the fundamental limits of 1-bit mean estimation for broader, non-parametric distribution classes (such as those defined solely by bounded variance) largely unresolved.

\subsection{Problem Setup}
\label{sec: problem setup}
\textbf{Distributional assumption.}
Let $X$ be a real-valued random variable\footnote{Our results also have implications for certain multivariate settings; see Section \ref{sec:multivar} for details.} with unknown distribution $D$. 
We assume that $D$ belongs to a (non-parametric) family $\mathcal{D} =  \mathcal{D}(\lambda, \sigma)$, defined by known parameters $\lambda \ge \sigma > 0$; a distribution $D$ is in this family if the following conditions hold:
 \begin{enumerate}[topsep=0pt, itemsep=0pt]
     \item Bounded mean: $\mu(D)  \in [-\lambda, \lambda]$,\footnote{Without loss of generality, we set the search range to be symmetric.  Note that a dependence on the search range $\lambda$ is unavoidable in the 1-bit setting (see Theorem~\ref{thm: adaptive lower bound}), but a crude upper bound can be used due to the mild logarithmic dependence in the sample complexity (see Theorem~\ref{thm: main}).}
     \item Bounded variance: $\Var(X) \le \sigma^2 \le \lambda^2$,
 \end{enumerate}
 where both $\lambda$ and $\sigma$ are known to the learner.  Note that the support of $D$ may be unbounded. 

\textbf{1-bit communication protocol.}
The learner is interested in estimating the population mean $\mu = \mu(D) = \EE[X]$ from~$n$ independent and identically distributed (i.i.d.) samples $X_1, \dotsc, X_n \sim D$, subject to a 1-bit communication constraint per sample.
The estimation proceeds through an interactive protocol between a learner and a single memoryless agent\footnote{Equivalently, this can be viewed as a sequence of memoryless agents where the agent in each round may be different.  In particular, the agent in round $t$ only has access to $X_t$ and not to the previous samples $X_1,\dotsc,X_{t-1}$.} that observes i.i.d. samples and sends 1-bit feedback to the learner.
Specifically, for $t = 1, \dotsc, n$:
\begin{enumerate}[topsep=0pt, itemsep=0pt]
    \item The learner sends a 1-bit quantization function $Q_t\colon \mathbb{R} \to \{0, 1\}$ to an agent;

    \item The agent observes a fresh sample $X_t \sim D$ and sends a 1-bit message $Y_t = Q_t(X_t)$ to the learner.
\end{enumerate}
After $n$ rounds, the learner forms an estimate $\hat{\mu}$ based on the entire interaction history $\big(Q_1, Y_1, \dotsc, Q_n, Y_n \big)$. This (and similar) setting was also adopted in previous communication-constrained learning works, e.g.,~\cite{hanna2022solving, mayekar2023communication, lau2025quantile}.

The learner's algorithm in this protocol is formally defined as follows:
\begin{definition}[1-bit mean estimator]
A 1-bit mean estimator is an algorithm for the learner that operates within the above communication protocol. It consists of
\begin{enumerate}[topsep=0pt, itemsep=0pt]
    \item A (potentially randomized) query strategy for selecting the quantization functions $Q_1, \dotsc, Q_n$, where the choice of $Q_t$ can depend adaptively on the history of interactions $(Q_1, Y_1, \dotsc, Q_{t-1}, Y_{t-1})$.

    \item An estimation rule that maps the full transcript $(Q_1, Y_1, \ldots, Q_n, Y_n)$ to a final estimate $\hat{\mu} \in \mathbb{R}$.
\end{enumerate}
    We say that an estimator is \emph{non-adaptive} if the query strategy selects all quantization functions in advance, without access to any of the outcomes~$Y_1, \dotsc, Y_n$.
\end{definition}

\textbf{Interval query model.} In the problem formulation, we placed no restriction on the choice of quantization function $Q_t$.  However, motivated by the desire for ``simple'' choices in practice, we focus primarily on \emph{interval queries}, which take the form ``Is $X_t \in I_t?$'' for some interval $I_t = [a_t,b_t]$ (possibly with $a_t = -\infty$ or $b_t = \infty$). The resulting 1-bit feedback $Y_t$ is the corresponding binary answer $\mathbf{1}\{X_t \in I_t \}$.
Our main estimator will only use such queries, though we will also present a variant that uses general 1-bit queries.

\textbf{Learner's goal.}
The learner's goal is to design a 1-bit mean estimator that returns an accurate estimate with high probability, while using as few samples as possible. We formalize this notion as follows:

\begin{definition}[$(\epsilon, \delta)$-PAC]
     A mean estimator is $(\epsilon, \delta)$-PAC for distribution family $\mathcal{D}$ with sample complexity~$n(\epsilon,\delta)$ if, for each distribution $D \in  \mathcal{D}$, it returns an $\epsilon$-correct estimate $\hat{\mu}$ with probability at least $1-\delta$, i.e.,
    \begin{equation*}
        \text{for each } D \in \mathcal{D}, \quad
        \mathrm{Pr}\left(  |\hat{\mu} - \mu(D)| \le \epsilon \right) \ge 1- \delta
    \end{equation*}
    and the number of samples required is bounded by $n(\epsilon,\delta)$. The probability is taken over the samples $X_1, \ldots, X_n$ and any internal randomness of the estimator.
\end{definition}

\subsection{Summary of Contributions}
\label{sec: contributions}
With the problem setup now in place, we summarize our main contributions as follows:

\begin{itemize}[topsep=0pt, itemsep=0pt]
    \item We propose a novel adaptive 1-bit mean estimator (see Section~\ref{sec: algorithm}) that only makes use of interval queries.

    \item We show that the mean estimator is $(\eps, \delta)$-PAC for distribution family $\mathcal{D}(\lambda, \sigma)$, with a sample complexity that  matches the minimax lower bound  $\Omega(\sigma^2/\epsilon^2 \cdot \log(\delta^{-1}))$ for the unquantized setting up to logarithmic factors and an additional $\log(\lambda/\sigma)$ term (see Theorem~\ref{thm: main}). Our sample complexity bound scales logarithmically with $\lambda/\sigma$, which contrasts with existing bounds for communication-constrained non-parametric mean estimators scaling at least linearly in~$\lambda$.

    \item We derive a worst-case lower bound, showing that the additional $\log(\lambda/\sigma)$ term is unavoidable (see Theorem~\ref{thm: adaptive lower bound}).
    For the interval-query model, we establish an ``adaptivity gap'' by showing a worst-case lower bound $\Omega(\lambda \sigma/\epsilon^2 \cdot \log(\delta^{-1}))$ for non-adaptive estimators.

    \item We provide several extensions including improved logarithmic factors under stronger tail decay, handling partially unknown parameters $(\epsilon,\sigma)$, and a two-stage variant under general 1-bit queries.
    
\end{itemize}

\subsection{Related Work}
\label{sec: related work}
The related work on distributed mean estimation
is extensive, we only provide a brief outline here, emphasizing the most closely related works.

\textbf{Classical mean estimation.}
Mean estimation (in the unquantized setting) is a fundamental and well-studied problem in statistics, e.g., see~\cite{lee2022optimal, cherapanamjeri2022optimal, minsker2023efficient, dang2023optimality, gupta2024beyond} and the references therein.
The state-of-the-art $(\eps, \delta)$-PAC estimator by~\cite{lee2022optimal}  achieves a tight sample complexity $n = \left(2+o(1) \right) \cdot (\sigma^2/\eps^2) \cdot \log(1/\delta)$ for all distributions with finite variance~$\sigma$. These results serve as a natural benchmark for mean estimation problems under communication constraints.

\textbf{Distributed estimation and learning}.
Early work in distributed estimation, learning, and optimization was motivated by the applications of wireless sensor networks (see~\cite{xiao2006distributed, varshney2012distributed, veeravalli2012distributed, he2020distributed} and the references therein), with a recent resurgence driven by the rise of large-scale machine learning systems. This has led to the characterization of the sample complexity or minimax risk/error for various distributed estimation problems~\citep{ZhangDJW13, GargMN14, Shamir14, BravermanGMNW16, XuR18, HanMOW18, HanOW18, BarnesHO19, BarnesHO20, AcharyaCT20a, AcharyaCT20b, acharya2021distributed, acharya2021interactive, acharya2021estimating, acharya2023unified, shah2025generalized}.

While abundant, most of the existing literature differs in major aspects including the estimation goal itself, the use of parametric models, and/or imposing significantly stronger assumptions.  
 To our knowledge, none of the existing work on non-parametric distributed estimation captures our problem setup.
For example, distributed non-parametric density estimation~\cite{BarnesHO20, AcharyaCST21a} is an inherently harder problem, and accordingly the authors impose certain regularity conditions on the density function (e.g., belonging to Sobolev space).  Similarly, distributed non-parametric function estimation problems in~\cite{zhu2018distributed, Szab2018AdaptiveDM, Szab2020DistributedFE, cai2022distributednonparametric, zaman2022distributed} assume certain tail bounds on the likelihood ratio (e.g., Gaussian white noise model).

\textbf{Distributed mean estimation (DME).}
Several works study variants of mean estimation under communication constraints directly. A large body of work focuses on parametric settings, often assuming a known location-scale family \cite{kipnis2022mean, kumar2025one} with a particular emphasis on Gaussians \cite{ribeiro2006bandwidth1, cai2022distributed, cai2024distributed}.  These estimators crucially rely on CDF inversion, which are highly dependent on exact knowledge of the parametric family, and are not suitable for our non-parametric setting. 
The non-parametric mean estimators in~\cite{luo2005universal, ribeiro2006bandwidth2} can handle broader distributional families but require additional assumptions such as bounded support and/or smooth density functions. Furthermore, some of these estimators require more than 1 bit of feedback (per coordinate) per sample. 
A recent independent work on non-adaptive 1-bit mean estimation \cite{abdalla2026robust} partially circumvents these restrictive assumptions.  However, their estimator adopts a fixed quantization range whose width scales as $\Omega(\sigma^2/\epsilon)$ in the worst-case,\footnote{To bound the worst-case truncation bias by $O(\epsilon)$ under only a finite-variance assumption, it can be shown that one must set the range to be $\Omega(\sigma^2/\epsilon)$ due to the worst-case tightness of Cantelli's inequality (a one-sided version of Chebyshev's inequality).} and this translates to a sample complexity of $\Omega(\sigma^4 / \epsilon^4)$.
In contrast, our adaptive 1-bit mean estimator achieves near-optimal $\widetilde{O}(\sigma^2/\epsilon^2)$ rates for all distributions whose first two moments lie within known bounds.

\textbf{Empirical vs. population mean estimation.}
A closely related line of work focuses on distributed \textit{empirical} mean estimation of a fixed dataset, which is a key primitive in federated learning~\cite{suresh2017distributed, konevcny2018randomized, davies2020new, vargaftik2021drive, mayekar2021wyner, vargaftik2022EDEN, benbasat2024accelerating, babu2025unbiased}.
These estimators typically achieve a minimax optimal mean squared error (MSE) that scale as $\EE[ (\hat{\mu} - \mu_{\text{emp}})^2 ] = O(\lambda^2/n)$.
By using Markov's inequality and the median-of-means method, they can be converted to $(\eps, \delta)$-PAC \textit{population} mean estimator with a sample complexity of $n = \widetilde{O}(\lambda^2/\eps^2 \cdot \log(1/\delta))$.
In contrast, our mean estimator achieves a sample complexity of $\widetilde{O}(\sigma^2/\eps^2 \cdot \log(1/\delta) + \log(\lambda/\sigma))$, which is considerably smaller when $\sigma^2 \ll \lambda^2$. 
Although some empirical mean estimators achieve MSE that depends on empirical deviation/variance $\sigma_{\text{emp}}$ of the fixed dataset \cite{ribeiro2006bandwidth2, suresh2022correlated}, they require a bounded support. 
Furthermore, their MSE scale at least linearly with $\lambda$, e.g., 
the one in~\cite{suresh2022correlated} scales as $\EE[ (\hat{\mu} - \mu_{\text{emp}})^2 ] = O(\sigma_{\text{emp}} \lambda/n + \lambda^2/n^2)$. 
Consequently, converting them to $(\eps, \delta)$-PAC \textit{population} mean estimator using standard techniques would result in a sample complexity bound that scales at least linearly with~$\lambda$.

\section{ESTIMATOR AND UPPER BOUND}
In this section, we introduce our 1-bit mean estimator and provide its performance guarantee.  Our estimator is designed as a target-accuracy driven procedure that takes parameters $(\lambda,\sigma,\epsilon,\delta)$ as input.  It operates to ensure the desired accuracy $\eps$ is attained with probability at least $1-\delta$ while minimizing the sample complexity $n$, and hence does not have an explicit pre-specified sample budget.  However, the estimator can readily be applied to the fixed-budget setting where the sampling budget is given and the goal is to minimize the estimation error $\eps$. In Section~\ref{sec: unknown eps}, we address a harder variant of this, where $n$ is fixed but unknown to the learner.

\subsection{Description of the Estimator}
\label{sec: algorithm}
Our estimator first localizes an interval $I$ of length $O(\sigma)$ containing the mean $\mu$ with high probability (see Step~1). 
Using the mid-point of $I$ as the ``centre", it partitions~$\mathbb{R}$ into symmetric regions with width growing exponentially (see Step~2). 
The ``outer regions'' have low probability mass, from which we can infer that suffices to only estimate the contributions of ``significant'' regions (See Step~3). 
Finally, the estimator forms an estimate of the mean contribution from the significant regions to within an additive error of $\eps/2$ (see Steps~4--5).
This high-level strategy of performing ``localization'' (coarse estimation) and ``refinement'' (finer estimation) has also appeared in prior works such as \cite{cai2022distributed}, but with very different details, particularly for refinement (see Remark~\ref{rem: refinement novelty}).

In more detail, our mean estimator is outlined as follows, with any omitted details deferred to Appendix \ref{appendix: proof of main result}: 
\begin{enumerate}
    \item Using existing median estimation techniques, we localize a high probability confidence interval $[L, U]$ containing the median $M$ using
     \begin{equation}
         n_{\text{loc}}(\delta, \lambda, \sigma) = \Theta\left( \log \frac{\lambda}{\sigma} + \log \frac{1}{\delta} \right)
     \end{equation}
     1-bit threshold queries (which is a special case of interval queries).
     Using the well-known property $|\EE[X] - M| \le \sigma$, we have the high probability confidence interval $[L-\sigma, U+\sigma]$ containing the mean. It can then be verified that $|(U+\sigma) - (L -\sigma) | \le 6 \sigma$. Without loss of generality, we may assume that the interval $[L-\sigma, U+\sigma]$ is of length $6\sigma$ with the midpoint being exactly $0$, i.e., $L+U = 0$.

    \item We partition $\mathbb{R}$ into non-overlapping symmetric regions $R_1, R_{-1}, R_2, R_{-2}, \ldots$ with width growing exponentially as follows:
    \begin{equation}
    \label{eq: general R_i def}
        \frac{R_i}{\sigma} = 
        \begin{cases}
         \left[m_{i-1} , m_i  \right) & \text{if }  i \ge 1
         \\ \\
        - R_i /\sigma& \text{if } i \le -1,
        \end{cases}
    \end{equation}
    where\footnote{We choose to define $m_i$ as in~\eqref{eq: general m_i} for the convenience
    of analysis later on. Any exponential/geometric growth rate (e.g., $m_i = \Theta(a^i)$ for $a > 1$) would be sufficient to achieve the sample complexity in Theorem~\ref{thm: main}.  
    }  
    \begin{equation}
    \label{eq: general m_i}
        m_i= 
        \begin{cases}
        0 & \text{if }  i = 0
        \\ \\
         2^i & \text{if }  1 \le i \le  4
        \\ \\
        2(m_{i-1} -3 )   & \text{if }  i \ge 5.
        \end{cases}
    \end{equation}
    Note that $m_i = \Theta( 2^i)$ increases exponentially. 
    Since the sum of all $\mu_i = \EE \left[ X \cdot \mathbf{1} \left(X \in R_i \right)  \right]$ is $\mu$, we can consider estimating each $\mu_i$ separately.

    \item We identify a threshold $i_{\max} = \Theta\left( \log \left(\sigma/\epsilon \right) \right)$ such that the sum of $\mu_i$ for all $i$ satisfying $|i| > i_{\max}$ has at most $\eps/2$ contribution to $\mu$:
    \[
     \left| \sum_{|i| > i_{\max}} \mu_i \right| \le \frac{\eps}{2},
    \]
    so that they can be estimated as being 0, i.e., $\hat{\mu}_i = 0$ for $|i| > i_{\max}$. It remains to form an estimate $\hat{\mu}  = \sum_{|i| \le i_{\max}} \hat{\mu}_i$ satisfying
    \[
    \left| \sum_{|i| \le i_{\max}} \mu_i - \sum_{|i| \le i_{\max}} \hat{\mu}_i \right| \le  \frac{\eps}{2}
    \]
    as this implies
    \begin{align*}
        \left| \mu -  \hat{\mu} \right| 
        &=
        \left| 
            \sum_{|i| \le i_{\max}} \mu_i  +  \sum_{|i| > i_{\max}} \mu_i  - \sum_{|i| \le i_{\max}} \hat{\mu}_i  
        \right| \\
        &\le 
        \left| \sum_{|i| \le i_{\max}} \mu_i  - \sum_{|i| \le i_{\max}} \hat{\mu}_i  \right| +  \left| \sum_{|i| > i_{\max}} \mu_i \right| \\
        &\le  \frac{\eps}{2} +  \frac{\eps}{2} \\ &= \eps.
    \end{align*}

    \item  
    Let $R_i = [a_i, b_i)$. and define the random variable $T_i \sim \mathrm{Unif}(a_i, b_i)$. For ease of notation, we write $p_{a_i} \coloneqq \mathrm{Pr}\left(X \in [a_i, T_i] \right)$
    and $p_{b_i} \coloneqq \mathrm{Pr}\left(X \in [T_i, b_i] \right)$.
    It can be verified that
    \begin{equation}
    \label{eq: mu_i = a_i Pr(X in [a_i, T_i]) + b_i Pr(X in [T_i, b_i])}
        \begin{aligned}
        \mu_i 
        = \EE \left[ X \cdot \mathbf{1} \left(X \in R_i \right)  \right]
        = a_i \cdot p_{a_i} 
         + b_i \cdot p_{b_i}.
    \end{aligned}
    \end{equation}
    In Appendix~\ref{appendix: SQ}, we show that $p_{a_i}$ (resp. $p_{b_i}$)
    is equivalent to the probability of $X$ being in $R_i$ and getting rounded down to $a_i$ (resp. rounded up to $b_i$) by a binary stochastic quantizer for $R_i$.   

    \item 
    It is therefore sufficient for the learner to form good estimates $\hat{p}_{a_i}$ of $p_{a_i}$ and $\hat{p}_{b_i}$ of $p_{b_i}$ such that
    \[
        \hat{\mu}_i =  a_i \cdot \hat{p}_{a_i} 
         + b_i \cdot \hat{p}_{b_i}
    \]
    satisfies
    \begin{align*}
        &\left| \sum_{|i| \le i_{\max}} \mu_i - \sum_{|i| \le i_{\max}} \hat{\mu}_i \right| \\
        = 
        &\left|
         \sum_{|i| \le i_{\max}} 
         a_i \cdot ( p_{a_i} -  \hat{p}_{a_i} ) + b_i \cdot ( {p}_{b_i} - \hat{p}_{b_i} ) 
        \right| \\
        \le \ & \frac{\eps}{2}.
    \end{align*}
    The learner estimates them separately using randomized interval queries of the form $\mathbf{1}\{X \in [a_i, T_i]\}$ and $\mathbf{1}\{X \in [T_i, b_i]\}$, and empirical averages of the 1-bit feedback sent by agent.  
    Using standard concentration inequalities, the number of 1-bit observations~$n_i$ needed to form ``accurate'' estimates can be bounded by
    $n_i =  \widetilde{O} \left( \left(  \frac{\sigma^2}{\epsilon^2}   + \frac{2^{i} \sigma}{\epsilon} \right) \cdot \log\left( \frac{1}{\delta}\right)\right)$. Summing over all $i$ satisfying $|i| \le i_{\max} = \Theta\left( \log \left(\sigma/\epsilon \right) \right)$ gives
    \[
     n_{\text{ref}}(\eps, \delta, \sigma) = \sum_{i: |i| \le i_{\max}} n_i  =  
     \widetilde{O} \left(  \frac{\sigma^2}{\epsilon^2}  \cdot \log\left( \frac{1}{\delta}\right)\right).
    \]
    Combining this with the $n_{\text{loc}}(\delta, \lambda, \sigma) = O\left( \log \frac{\lambda}{\sigma} + \log \frac{1}{\delta} \right)$ samples used in Step 1 for localization, we obtain a total sample complexity of
    \begin{equation*}
        n \coloneqq n(\eps, \delta, \lambda, \sigma) = \widetilde{O} \left(  \frac{\sigma^2}{\epsilon^2}  \cdot \log\left( \frac{1}{\delta}\right) + \log\frac{\lambda}{\sigma}\right). 
    \end{equation*}
\end{enumerate}

\begin{remark}
\label{rem: refinement novelty}
   Standard refinement components such as stochastic quantization or CDF inversion fail to achieve the near-optimal variance-dependent rate in the general non-parametric setting. To achieve the near-optimal rate, we introduce specific design choices in the refinement stage that interact delicately:
    \begin{itemize}[topsep=0pt, itemsep=0pt]
        \item an exponential interval partitioning scheme that concentrates sampling on regions with ``significant'' probability mass;
        \item a carefully chosen truncation index that balances (a) bias from tail truncation and (b) sample complexity blow up from covering too many “insignificant” regions;
        \item a calibrated region-wise sample allocation to ensure the near-optimality of the “global” sample complexity. 
    \end{itemize}
    Eliminating or simplifying any of these components is likely to ``break'' the claimed sample complexity given below.
\end{remark}

\subsection{Upper bound}
We now formally state the main result of this paper, which is the performance guarantee of our mean estimator in Section~\ref{sec: algorithm}. The proof is deferred to Appendix~\ref{appendix: proof of main result}, where we also provide the omitted details in the above outline.

\begin{theorem}
\label{thm: main}
    The mean estimator given in Section~\ref{sec: algorithm} is
    $(\eps, \delta)$-PAC for distribution family $\mathcal{D}(\lambda, \sigma)$, with sample complexity
    \begin{align}
    n 
    &= O \left(  \frac{\sigma^2}{\epsilon^2} \cdot  \log^3\left( \frac{\sigma}{\epsilon}\right)  \cdot \log\left( \frac{ \log\left( \frac{\sigma}{\epsilon}\right) }{\delta}\right) + \log\frac{\lambda}{\sigma}\right)
    \label{eq: scaling law of n} \\
    & = \widetilde{O}\left( \frac{\sigma^2}{\epsilon^2}\log\frac{1}{\delta} + \log\frac{\lambda}{\sigma}\right).
\end{align}
\end{theorem}
Thus, we match the unquantized scaling up to logarithmic factors (see Section~\ref{sec: related work}) and an additional $O(\log \lambda/\sigma)$ term. 
In Theorem~\ref{thm: adaptive lower bound} below, we show that this $\log (\lambda/\sigma)$ term is unavoidable.
In Sections~\ref{sec: finite k} and \ref{sec: subgaussian}, we provide improved upper bounds for distributions with stronger tail decay.
We also study variants where $(\epsilon,\sigma)$ are not prespecified in Sections~\ref{sec: unknown eps} and~\ref{sec: unknown variance} and a variant that uses only two rounds/stages of adaptivity in Section~\ref{sec: two-stage}.

\begin{remark}[Computational Complexity]
    While our primary focus is sample complexity, the proposed estimator is also computationally efficient. The localization step (Step 1) relies on noisy binary search subroutine, requiring $O(\log^2(\lambda/\sigma))$ time~\cite[Section 6]{gretta2023sharp}. The refinement step (Steps 2–5) operates in time linear in the sample complexity, assuming constant-time arithmetic operations and random threshold generation.
\end{remark}

\section{LOWER BOUND AND ADAPTIVITY GAP}
In this section, we provide two lower bounds on the sample complexity. We first provide, in Theorem~\ref{thm: adaptive lower bound}, a near-matching worst-case lower bound to the upper bound in Theorem~\ref{thm: main}. In particular, we show that the $\log(\lambda/\sigma)$ term is unavoidable. Perhaps more interestingly, we show in Theorem~\ref{thm: non-adaptive lower bound} that the best non-adaptive mean estimator is strictly worse than our adaptive mean estimator, at least under the interval query model. This shows that there is an ``adaptivity gap'' between the performance of adaptive and non-adaptive interval query based mean estimators. The proofs are given in Appendix~\ref{appendix: Lower Bounds}.

\begin{theorem}
\label{thm: adaptive lower bound}
For any $(\epsilon, \delta)$-PAC 1-bit mean estimator, and any $\epsilon < \sigma/2$, there exists a distribution $D \in \mathcal{D}(\lambda, \sigma)$ such that the number of samples~$n$ must satisfy
\[
    n = \Omega\left( \frac{\sigma^2}{\epsilon^2} \cdot \log\left(\frac{1}{\delta} \right) + \log\left( \frac{\lambda}{\sigma} \right) \right).
\]
\end{theorem}

\begin{theorem}
\label{thm: non-adaptive lower bound}
    For any non-adaptive $(\epsilon, \delta)$-PAC estimator that only makes interval queries, and any $\epsilon < \sigma/2$, there exists a distribution $D \in \mathcal{D}(\lambda, \sigma)$ such that the number of samples $n$ must satisfy
    \[
        n = \Omega
        \left( 
            \frac{\lambda \sigma}{\epsilon^2} \cdot \log\left(\frac{1}{\delta} \right)
        \right).
    \]
\end{theorem}

Both proofs are based on constructing a (finite) ``hard subset'' of distributions that capture two sources of difficulty: (i) ``coarsely'' identifying the distribution's location in $[-\lambda,\lambda]$ among $\Theta(\lambda/\sigma)$ possibilities, and (ii) ``finely'' estimating the mean by distinguishing between two possibilities at that location whose means differ by $2\epsilon$. The fine estimation step inherently requires $\Omega(\sigma^2/\eps^2 \cdot \log(1/\delta))$ samples, based on standard hypothesis testing lower bound. However, the dependency on $\lambda/\sigma$ arising from the coarse identification step differs in adaptive vs. non-adaptive settings:
\begin{itemize}[topsep=0pt, itemsep=0pt]
    \item  In Theorem~\ref{thm: adaptive lower bound} (adaptive setting), we can simply interpret the additive logarithmic dependence as the number of bits needed to identify the correct location among the $\Theta(\lambda/\sigma)$ possibilities, with each query giving at most 1 bit of information. 
    
    \item  In Theorem \ref{thm: non-adaptive lower bound} (non-adaptive setting), the \emph{multiplicative} dependence arises because the estimator needs to allocate enough queries in \emph{every} one of the $\Theta(\lambda/\sigma)$ locations, as it does not know the correct location in advance.
\end{itemize}

We note that the distributed Gaussian mean estimator in~\cite{cai2024distributed} is non-adaptive and achieves an order-optimal MSE. However, their estimator is specific to Gaussian distributions, and their quantization functions are not based on interval queries.  We will build on their localization strategy in our two-stage variant (Section \ref{sec: two-stage}), but we avoid their refinement strategy which is much more Gaussian-specific (CDF inversion).

\section{VARIATIONS AND REFINEMENTS}
\label{sec: variations and refinements}

\subsection{Bounded Higher-Order Moments}
\label{sec: finite k}

Suppose further that the random variable $X$ has a finite $k$-th central moment bounded by $\sigma^k$ for some $k > 2$, i.e.,
\begin{equation}
\label{eq: bounded finite moment}
    \EE\big[ |X - \mu |^k \big] \le \sigma^k.
\end{equation}
By Lyapunov's inequality, we have
\[
    \left( \EE\big[ |X - \mu |^2 \big] \right)^{1/2} \le \left( \EE\big[ |X - \mu |^k \big] \right)^{1/k} = \sigma,
\]
which implies that the variance is bounded by $\sigma^2$. The condition \eqref{eq: bounded finite moment} imposes a stronger tail decay that is imposed by variance alone,\footnote{Note that if the learner knows  $\Var(X) \le \gamma^2$ for some known $\gamma \ll \sigma$, then using our main estimator in Section~\ref{sec: algorithm} would still lead to a better sample complexity.  That is, this refinement is primarily of interest when $\gamma$ and $\sigma$ are comparable.} and in this case we can tighten the $\log(\sigma/\epsilon)$ factors in Theorem~\ref{thm: main} to $\log_{k/2} \log(\sigma/\epsilon)$.
\begin{theorem}
\label{thm: higher-order variant}
    Suppose that random variable $X$ satisfies~\eqref{eq: bounded finite moment} for some $\sigma >0$ and $k > 2$. Then there exists an $(\eps, \delta)$-PAC 1-bit mean estimator with sample complexity
   \begin{align*}
     n = O \bigg(  &\frac{\sigma^2}{\epsilon^2} \cdot   \left( \log_{k/2} \log \left( \frac{\sigma}{\epsilon} \right)\right)^3   \cdot \log\left( \frac{\log_{k/2} \log \left( \frac{\sigma}{\epsilon}\right) }{\delta}\right) \\
    + &\log\frac{\lambda}{\sigma}\bigg).
    \end{align*}
\end{theorem}
The protocol and proof of its guarantee are similar to those given in Section~\ref{sec: algorithm}, with the main difference being that we change $m_i$ in~\eqref{eq: general m_i} to a choice that scales doubly exponentially (see~\eqref{eq: finite k m_i} in Appendix~\ref{appendix: finite k}). Consequently, $i_{\text{max}}$ scales as $\log_{k/2} \log(\sigma/\eps)$. The details are given in Appendix~\ref{appendix: finite k}.

\subsection{Sub-Gaussian Random Variables}
\label{sec: subgaussian}
Now we suppose that $X - \mu$ is sub-Gaussian with known parameter $\sigma^2$, i.e.,
\begin{equation}
\label{eq: subgaussian bound}
\begin{aligned}
    \mathrm{Pr}(|X - \mu| \ge t) &\le 2 \exp \left( -\frac{t^2}{2 \sigma^2} \right).
\end{aligned}
\end{equation}
Note that we have $\Var(X) \le \sigma^2$ and $X$ has a finite $k$-th central moment for every~$k$.
In this case, we can tighten the $\log_{k/2} \log(\sigma/\epsilon)$ factors in Theorem~\ref{thm: higher-order variant} to $\log^*(\sigma/\epsilon)$, where the ``iterated logarithm" $\log^*(\cdot)$ is the number of times the logarithm function must be iteratively applied before the result is less than or equal to 1.
\begin{theorem}
\label{thm: subgaussian variant}
    Suppose that $X - \mu$ is sub-Gaussian with known parameter $\sigma^2$.
    Then there exists an $(\eps, \delta)$-PAC 1-bit mean estimator
     with sample complexity
   \begin{equation*}
    n = O \left(  \frac{\sigma^2}{\epsilon^2} \cdot   \left( \log^*\left( \frac{\sigma}{\epsilon} \right)\right)^3   \cdot \log\left( \frac{ \log^*\left( \frac{\sigma}{\epsilon}\right) }{\delta}\right) + \log\frac{\lambda}{\sigma}\right).
    \end{equation*}
\end{theorem}
The protocol and proof of its guarantee are again similar to those of Theorem~\ref{thm: main}, with the main difference being that we change $m_i$ in~\eqref{eq: general m_i} to a choice that scales according to a tower of exponentials of height~$i$. Consequently, $i_{\text{max}}$ scales as $\log^*(\sigma/\eps)$. The details are given in Appendix~\ref{appendix: subgaussian}.

\subsection{Unknown Target Accuracy}
\label{sec: unknown eps}
By inverting the $\eps$ term in~\eqref{eq: scaling law of n}, we obtain a performance guarantee on the target accuracy of our main algorithm in terms of parameters $n, \delta, \sigma$, and $\lambda$, where $n$ is the pre-specified sampling budget. In other words, by running our mean estimator with parameters
\[
\left(  \eps, \delta, \lambda, \sigma\right) = 
\left(  \eps(n, \delta, \lambda, \sigma), \ \delta, \lambda, \sigma\right)
\] 
we obtain a mean estimate that is $\eps$-accurate with error probability at most $\delta$, where $\eps = \eps(n, \delta, \lambda, \sigma)$ is computed by inverting the $\eps$ term in~\eqref{eq: scaling law of n}. Furthermore, the number of samples used
\[
 n(\eps, \delta, \lambda, \sigma)  = n_{\text{loc}}(\delta, \lambda, \sigma) + n_{\text{ref}}(\eps, \delta, \lambda, \sigma)
\]
trivially satisfy the pre-specified sampling budget. Here we define
\[
n_{\text{loc}}(\delta, \lambda, \sigma) = \Theta \left( \log\frac{\lambda}{\sigma} + \log \frac{1}{\delta} \right)
\]
as the number of samples used in the localization step of our mean estimator (see Step 1 of Section~\ref{sec: algorithm}), and
\begin{equation}
\label{eq: n_ref}
    n_{\text{ref}}(\eps, \delta, \sigma) 
=  \Theta \left(  \frac{\sigma^2}{\epsilon^2} \cdot \log^3\left( \frac{\sigma}{\epsilon}\right) \cdot \log\left( \frac{ \log\left( \frac{\sigma}{\epsilon}\right) }{\delta}\right) \right)
\end{equation} as the number of samples used in Steps 3-5 (refinement).

We now consider the scenario where the true sample budget $n_{\text{true}}$ is not pre-specified in advance, but the other parameters $\sigma$, $\lambda$, and~$\delta$ are still known.
In this case, we can no longer run the algorithm using an $\eps$ inverted as before, since $n$ is not pre-specified. A naive way is to guess some $\eps_\text{guess}$ and run the estimator with parameters ($\eps_\text{guess}, \delta, \lambda, \sigma$).
Moreover, if we guess some $\eps_\text{guess}$ which ends up being too big, i.e., $\eps_\text{guess}  \gg \eps(n_{\text{true}}, \delta, \lambda, \sigma)$, then the resulting estimator is inaccurate given the sampling budget. Conversely, if the guess is too small, i.e., $\eps_\text{guess}  \ll \eps(n_{\text{true}}, \delta, \lambda, \sigma)$, then the sampling budget may not be sufficient to guarantee an $\eps_\text{guess}$-accurate estimate.

To overcome this, we can use a standard halving trick on $\eps_\text{guess}$ (along with careful consideration of $\delta$) to ``anytime-ify'' the mean estimator. We run the localization step and the partitioning step (see Steps 1-2 of Section~\ref{sec: algorithm}) once, which requires knowing only the parameters $(\delta, \lambda, \sigma)$ and uses $n_{\text{loc}}(\delta, \lambda, \sigma)$ samples. Then for each round $\tau = 1, 2, \dotsc$, we run Steps~3-5 of our mean estimator with parameters
\[
    ( \eps_{\tau}, \delta_{\tau}, \sigma)
    \quad \text{where}  \quad
   \eps_{\tau} = \frac{\sigma}{2^{\tau}} 
   \quad \text{and}\footnote{Alternatively, we could pick any suitable $\eps_0$ as the initial target accuracy and let $\eps_{\tau} = \eps_0/2^{\tau-1}$.}  \quad 
    \delta_{\tau} = \frac{6\delta}{\pi^2 \tau^2}.
\]
Note that this process would use $n_{\text{ref}}(\eps_{\tau}, \delta_{\tau}, \sigma)$ samples in each round $\tau$. It follows that round $\tau$ will complete as long as the true sampling budget $n_{\text{true}}$ satisfies
\[
   n_{\text{loc}}(\delta, \lambda, \sigma) + \sum_{s=1}^{\tau} n_{\text{ref}}(\eps_s, \delta_s, \sigma)
   \le 
    n_{\text{true}}. 
\]
When the real sampling budget $n_{\text{true}}$ is exhausted, we stop and output the last estimate we fully computed, i.e., we output $\hat{\mu}_T$,
where 
\begin{equation}
\label{eq: last round T}
     T = \max_{\tau \ge 1} \left\{ \sum_{s=1}^{\tau} n_{\text{ref}}(\eps_s, \delta_s, \sigma) \le 
    n_{\text{true}} - n_{\text{loc}}(\delta, \lambda, \sigma)\right\}
\end{equation}
is the last round where the subroutine is completed. 
By the union bound and the guarantee of the subroutine for each round $\tau$, we have with probability at least $1-\delta$ that every estimate $\hat{\mu}_{\tau}$ formed is $\eps_{\tau}$-accurate.
In particular, under this high-probability event, the final output $\hat{\mu}_T$ satisfies
\[
  |\hat{\mu}_T - \mu | \le \eps_T = \frac{\sigma}{2^{T}}.
\]
Ideally, we would like to compare $\eps_T$ against the ``oracle accuracy'' $\eps^*$, which satisfies
\[
    n_{\text{ref}}(\eps^*, \delta,  \sigma) = n_{\text{true}} - n_{\text{loc}}(\delta, \lambda, \sigma),
\]
i.e, $\eps^*$ is the optimal target accuracy that could be achieved (with high probability) had we known the unknown sampling budget $n_{\text{true}}$ in advance. 
Indeed, under a mild assumption of $n$ not being too small, we show that $\eps_T$ obtained from the doubling trick matches the ``oracle'' value to within a constant factor.

\begin{theorem}
\label{thm: unknown eps}
    Under the preceding setup, assuming $n_{\text{true}} \ge n_{\text{loc}}(\delta, \lambda, \sigma)$, we have $\eps_T = O(\eps^*)$.
\end{theorem}
The proof is given in Appendix~\ref{appendix: unknown target accuracy}.

\subsection{Adapting to Unknown Variance}
\label{sec: unknown variance}
The sample complexity of our mean estimator, as stated in Theorem~\ref{thm: main}, scales quadratically with $\sigma/\eps$, where $\sigma^2$ is a known upper bound on the true variance $\sigma_{\text{true}}^2 = \Var(X)$. This scaling is not ideal when the upper bound is loose. This is in contrast to the unquantized setting, where there exist mean estimators whose sample complexity scales quadratically with $\sigma_{\text{true}}/\eps$ without any knowledge of $\sigma$ \cite{lee2022optimal}.

Under the 1-bit communication constraint, it may be difficult to learn the true variance (and estimate mean at the same time).
We consider the case where both target accuracy $\eps$ and true variance $\sigma_{\text{true}}$ are unknown, but we know that
\[
\sigma_{\text{true}} \in [\sigma_{\text{min}}, \sigma_{\text{max}}]
\quad \text{and} \quad
\eps =  r \sigma_{\text{true}}
\]
for some known $r$. That is, we seek accuracy to within $r$ multiples of standard deviation, even though we do not know standard deviation~$\sigma_{\text{true}}$.

In this case, we construct a mean estimator that uses the mean estimator in Section~\ref{sec: algorithm} as subroutine. 
Set
\begin{equation}
    \label{eq: T = log(sigma_max/sigma_min)}
    T = \left\lceil \log_2 \left(\sigma_{\text{max}} / \sigma_{\text{min}} \right) \right\rceil.
\end{equation} 
For $i=0, \dots, T$, we define
\begin{equation}
\sigma_i = \sigma_{\text{min}} \cdot 2^i
    \iff
    \label{eq: sigma_i}
    \sigma_i =
    \begin{cases} 
    \sigma_{\text{min}} & \text{ if } i = 0
    \\    
  2 \sigma_{i-1}    & \text{ if } 1 \le  i \le T
\end{cases}, 
\end{equation}
and run the mean estimator in Section~\ref{sec: algorithm} with parameters
\begin{equation}
    \label{eq: eps_i = r sigma_i/5}
    \left(\eps_i, \delta_i, \lambda, \sigma_i \right)=
    \left( 
     \frac{r \sigma_i}{5},
    \quad \frac{\delta}{T+1},
    \quad
    \lambda, 
    \quad
     \sigma_i,
    \right)
\end{equation}
to obtain an estimate $\hat{\mu}^{(i)}$.
For each $i$, we define a confidence interval 
\begin{equation}
\label{eq: CI}
    I_i = [ \hat{\mu}^{(i)} \pm \eps_i]
    = \left[ \hat{\mu}^{(i)} -  \frac{r \sigma_i}{5}, \hat{\mu}^{(i)} + \frac{r \sigma_i}{5} \right]
\end{equation}
of length $2 \eps_i$, and we say that $\sigma_i$ is \emph{feasible} if $I_i$ overlaps with all confidence intervals of higher indices:
\begin{equation}
\label{eq: feasible condition}
     I_i \cap I_j \ne \emptyset  \quad \text{for all } j > i.  
\end{equation}
Note that $\sigma_T$ is trivially feasible.
We return the mean estimate corresponding to the smallest feasible $\sigma_i$, i.e., we return the estimate $\hat{\mu}^{(i^*)}$ where
$i^*$ is the smallest $i$ that satisfies condition~\eqref{eq: feasible condition}. The resulting mean estimator has a sample complexity that scales quadratically with $\sigma_{\text{true}}/\eps = 1/r$, but pays an extra multiplicative factor $\log(\sigma_{\max}/ \sigma_{\min})$. 
\begin{theorem}
\label{thm: partially unknown variance}
    The mean estimator above is $(\eps, \delta)$-PAC with sample complexity
    \begin{equation*}
          n =
         \widetilde{O} \bigg(    
          \log \bigg(\frac{\sigma_{\text{max}}}{\sigma_{\text{min}}} \bigg) 
          \bigg( \frac{1}{r^2} \log\bigg( \frac{1}{\delta}\bigg) + \log\frac{\lambda}{\sqrt{\sigma_{\min} \sigma_{\max}}}      \bigg)
         \bigg).
    \end{equation*}
\end{theorem}
The proof is given in Appendix~\ref{appendix: partially unknown variance}.

\begin{remark}
    Intuitively, the feasibility condition~\eqref{eq: feasible condition} tells us whether an interval $I_i$ is consistent with the intervals obtained using larger/more conservative $\sigma$-values. In particular, if $\sigma_i \ge \sigma_{\text{true}}$ then $\sigma_i$ is feasible (see Appendix~\ref{appendix: partially unknown variance}), but the converse may not hold.
    In practice, we can start with the largest $\sigma$-value and sequentially half it (i.e., $\sigma_i  = \sigma_{\text{max}}/2^i$), until we find the first $i$ where $\sigma_i$ is infeasible, and return $\hat{\mu}^{(i+1)}$. Although this may not lead to an improvement in the upper bound (e.g., the loop may not terminate even when $\sigma_i < \sigma_{\text{true}}$), it can help avoid using all $T$ loops when it is unnecessary to do so.
\end{remark}

\subsection{Two-Stage Variant}
\label{sec: two-stage}
Our mean estimator in Section~\ref{sec: algorithm} uses $O\left( \log \frac{\lambda}{\sigma} + \log \frac{1}{\delta} \right)$ rounds of adaptivity. Specifically, the localization step (Step~1 of Section~\ref{sec: algorithm}), which performs median estimation through noisy binary search, requires $O\left( \log \frac{\lambda}{\sigma} + \log \frac{1}{\delta} \right)$ rounds of adaptivity; while the refinement step can be done in just one additional round after we have localized an interval of length $O(\sigma)$ containing the mean. In this section, we provide an alternative localization procedure that is non-adaptive, with the remaining steps unchanged. This gives us an alternative mean estimator that requires only two rounds of adaptivity -- one for localization and one for refinement. However, this comes at the cost of using general 1-bit queries in the first round, as opposed to only using interval queries.

Our alternative localization step is adapted from the localization step of the non-adaptive Gaussian mean estimator in~\cite{cai2024distributed}, which is presented therein for Gaussian distributions but also noted to extend to the general sub-Gaussian case (unlike their refinement stage).  We modify their localization step so that it works on all distributions with mean and variance lying within known bounds (namely, $[-\lambda,\lambda]$ and $[0,\sigma^2]$ respectively), with the following performance guarantee:

\begin{theorem}
\label{thm: alternative localization}
There exists a 1-bit non-adaptive localization protocol taking $(\delta, \lambda, \sigma)$ as input such that for each $D \in \mathcal{D}$, it returns an interval $I$ containing~$\mu$ with probability at least $1-\delta/2$. Furthermore, the number of samples used is 
$\Theta\left( \log \left( \frac{\lambda}{\sigma}\right) \cdot \log \frac{\log(\lambda/\sigma)}{\delta}\right)$ and $|I| = O(\sigma)$.
\end{theorem}

We describe the high-level idea here. The learner partitions the interval $[-\lambda, \lambda]$ into $2^K$ subintervals $\{I_0, I_1, \dotsc, I_{2^K-1} \}$ of same length for some $K = \Theta(\log (\lambda/\sigma))$, and the learner tries to estimate all $K$ bits of the Gray code representation of the subinterval containing~$\mu$. Each of these $K$ bits is estimated reliably by taking a majority vote over $J = \Theta\left(\log\frac{K}{\delta}\right)$ samples.
The details are given in Appendix~\ref{appendix: gray code}. 

By replacing the localization step of our main estimator (Step~1 of Section~\ref{sec: algorithm}) with the alternative localization step above, we have a mean estimator with the following performance guarantee.
\begin{corollary}
     The alternative mean estimator described above is
    $(\eps, \delta)$-PAC for distribution family $\mathcal{D}(\lambda, \sigma)$, with sample complexity
    \[
    n =\widetilde{O}\left( \frac{\sigma^2}{\epsilon^2}\log\frac{1}{\delta}  + \log \left( \frac{\lambda}{\sigma}\right) \cdot \log \log \left( \frac{\lambda}{\sigma}\right)\right).
    \]
    Furthermore, it uses only two rounds of adaptivity, the first of which uses general (non-interval) 1-bit queries.
\end{corollary}

\subsection{Multivariate Mean Estimation} \label{sec:multivar}

The multivariate case (i.e., $X \in \mathbb{R}^d$ with $d > 1$) is naturally of significant interest. We have focused on the univariate case since it is the natural starting point and is already challenging. However, our results turn out to also provide some preliminary findings for multivariate settings.

Specifically, suppose that $X$ takes values in $\mathbb{R}^d$ and has entries $X_1,\dotsc,X_d$ satisfying our earlier assumptions individually for each coordinate $i=1,\dotsc,d$.  By applying our univariate techniques coordinate-wise with parameters $\epsilon/\sqrt d$ and $\delta/d$, we obtain an overall estimate that is $\epsilon$-accurate in $\ell_2$ norm with probability at least $1-\delta$. In accordance with Theorem \ref{thm: main}, the sample complexity is
\[
\widetilde{O}\left( \frac{d^2 \sigma^2}{\epsilon^2}\log\frac{1}{\delta} + d\log\frac{\lambda}{\sigma}\right),
\]
where the $d^2$ factor arises from (i) using the scaled accuracy parameter $\epsilon/\sqrt{d}$, and (ii) running the univariate subroutine~$d$ times.  This may seem potentially loose on first glance, due to the correct scaling being $\frac{\sigma^2}{\epsilon^2} \cdot \left(d + \log(1/\delta) \right)$ in the absence of a communication constraint \cite{lugosi2019sub}.
However, under 1-bit feedback, the $d^2 \sigma^2 / \epsilon^2$ dependence in fact unavoidable even in the special case of Gaussian random variables; see \cite[Theorem 8]{cai2024distributed} with the parameter $m'$ therein equating to $n/d$ in our notation under 1-bit feedback.\footnote{To give slightly more detail, the parameters $m$ and $b_i = 1$ therein equate respectively to the number of samples $n$ and number of bits allowed per feedback.}  
Moreover, if the communication bottleneck is relaxed to allow $d$ bits of feedback per sample (i.e., one bit \emph{per coordinate}), applying our univariate estimator coordinate-wise yields a sample complexity of $\widetilde{O}\big( \frac{d \sigma^2}{\epsilon^2} \log(1/\delta) + d \log(\lambda/\sigma) \big)$. In the constant error probability regime ($\delta = \Theta(1)$), this matches the unconstrained rate up to logarithmic factors.   
In the regime $\delta = o(1)$ there remains a significant gap due to the fact that $d \log \frac{1}{\delta} \gg d + \log\frac{1}{\delta}$, but this gap is inherent to any approach that controls each coordinate's error to $O\big( \frac{\epsilon}{\sqrt d} \big)$ separately.
 
Beyond the issue of joint $(d,\delta)$ dependence, another limitation of the coordinate-wise approach is that it does not capture the dependence on off-diagonal terms in the covariance matrix $\Sigma$.  Doing so may be significantly more difficult, particularly when $\Sigma$ is not known exactly and so ``whitening'' techniques cannot readily be used.  We leave such considerations for future work.

\section{CONCLUSION}
In this paper, we studied the problem of estimating the mean of a distribution under the extreme constraint of a single bit of communication per sample.
We proposed an adaptive estimator that is $(\epsilon, \delta)$-PAC for all distributions with bounded mean and variance, which achieves near-optimal sample complexity. This result demonstrates that the statistical efficiency of mean estimation is largely preserved under 1-bit communication constraints. 
We also established an adaptivity gap for the interval query model, showing that non-adaptive strategies are strictly suboptimal.
Several directions remain for future research, including tightening the polylogarithmic factors, adapting to unknown variance and target accuracy with as few assumptions as possible, and extending to multivariate settings beyond the coordinate-wise approach.

\section*{Acknowledgment}
This work was supported by the Singapore National Research Foundation (NRF) under its AI Visiting Professorship programme.

\bibliography{bibliography}

\section*{Checklist}



\begin{enumerate}

  \item For all models and algorithms presented, check if you include:
  \begin{enumerate}
    \item A clear description of the mathematical setting, assumptions, algorithm, and/or model. [Yes]
    \item An analysis of the properties and complexity (time, space, sample size) of any algorithm. [Yes]
    \item (Optional) Anonymized source code, with specification of all dependencies, including external libraries. [Not Applicable]
  \end{enumerate}

  \item For any theoretical claim, check if you include:
  \begin{enumerate}
    \item Statements of the full set of assumptions of all theoretical results. [Yes]
    \item Complete proofs of all theoretical results. [Yes]
    \item Clear explanations of any assumptions. [Yes]     
  \end{enumerate}

  \item For all figures and tables that present empirical results, check if you include:
  \begin{enumerate}
    \item The code, data, and instructions needed to reproduce the main experimental results (either in the supplemental material or as a URL). [Not Applicable]
    \item All the training details (e.g., data splits, hyperparameters, how they were chosen). [Not Applicable]
    \item A clear definition of the specific measure or statistics and error bars (e.g., with respect to the random seed after running experiments multiple times). [Not Applicable]
    \item A description of the computing infrastructure used. (e.g., type of GPUs, internal cluster, or cloud provider). [Not Applicable]
  \end{enumerate}

  \item If you are using existing assets (e.g., code, data, models) or curating/releasing new assets, check if you include:
  \begin{enumerate}
    \item Citations of the creator If your work uses existing assets. [Not Applicable]
    \item The license information of the assets, if applicable. [Not Applicable]
    \item New assets either in the supplemental material or as a URL, if applicable. [Not Applicable]
    \item Information about consent from data providers/curators. [Not Applicable]
    \item Discussion of sensible content if applicable, e.g., personally identifiable information or offensive content. [Not Applicable]
  \end{enumerate}

  \item If you used crowdsourcing or conducted research with human subjects, check if you include:
  \begin{enumerate}
    \item The full text of instructions given to participants and screenshots. [Not Applicable]
    \item Descriptions of potential participant risks, with links to Institutional Review Board (IRB) approvals if applicable. [Not Applicable]
    \item The estimated hourly wage paid to participants and the total amount spent on participant compensation. [Not Applicable]
  \end{enumerate}

\end{enumerate}

\clearpage
\appendix
\thispagestyle{empty}
\onecolumn 
\aistatstitle{Supplementary Material}
\section{Proof of Theorem~\ref{thm: main} (Performance Guarantee of 1-bit Mean Estimator)}
\label{appendix: proof of main result}

We proceed in several steps as we outlined in Section~\ref{sec: algorithm}.

\textbf{Step 1 (Narrowing Down the Mean via the Median):}
We discretize the interval $[-\lambda, \lambda]$ containing $\EE[X]$
into a discrete set of points with uniform spacing of $\sigma$:\footnote{For ease of analysis, we assume that $\lambda$ is an integer multiple of $\sigma$.}
\begin{equation*}
    \left\{-\lambda, -\lambda + \sigma, \dotsc, -\sigma, 0, \sigma, \dotsc, \lambda-\sigma, \lambda\right\}.
\end{equation*}
We then form estimates $L, U \in  \left\{-\lambda, -\lambda + \sigma, \dotsc,  \lambda-\sigma, \lambda \right\}$ using noisy binary search~\cite{gretta2023sharp} that satisfy 
\begin{equation}
\label{eq: median lower bound}
    \mathrm{Pr}\left(
    \left[ F(L), F(L+\sigma) \right]
    \cap (0.49, 0.5) 
    \text{ is non-empty}
    \right) \ge 1- \delta
\end{equation}
and
\begin{equation}
\label{eq: median upper bound}
    \mathrm{Pr}\left(
    \left[ F(U - \sigma), F(U) \right]
    \cap (0.5, 0.51) 
    \text{ is non-empty}
    \right) \ge 1- \delta.
\end{equation}
The algorithm in~\cite{gretta2023sharp} achieves this using at most $O\big(\log \frac{\lambda}{\sigma \delta}\big)$ 1-bit queries.
Under these high-probability events, the median $M$ satisfies $L \le M \le U$.
Since $|\mu - M| \le \sigma$ (e.g., see \cite[Exercise 2.1]{boucheron2013concentration},
we have
\begin{equation*}
    \mu \in [L -\sigma, U +\sigma].
\end{equation*}
We would like to bound the length of the interval, $(U +\sigma) - (L -\sigma)$. 
To do so, we consider two different cases:
(i)  $L + \sigma \ge U - \sigma$ 
and (ii) $L + \sigma < U - \sigma$.
In case (i), the interval length is trivially at most $4 \sigma$.
In case (ii), we claim that the interval length is at most $6\sigma$.
Seeking contradiction, suppose the length of interval $(U +\sigma) - (L -\sigma) \ge 7\sigma$. Then we must have either 
\begin{equation*}
    \mu - (L -\sigma) \ge 3.5 \sigma
    \quad \text{or} \quad 
(U + \sigma) - \mu  \ge 3.5\sigma.
\end{equation*}
We will show that $\mu - (L -\sigma) \ge 3.5 \sigma$ (which implies $\mu - 1.5\sigma \ge L+ \sigma$) will lead to a contradiction; the case $(U + \sigma) - \mu \ge 3.5 \sigma$ is similar.
Using~\eqref{eq: median lower bound}, we have
\begin{equation*}
    \mathrm{Pr}\left(X \le \mu - 1.5 \sigma \right) \ge 
    \mathrm{Pr}(X \le L + \sigma) = F_X(L+ \sigma)   > 0.49.
\end{equation*}
On the other hand, by Chebyshev's inequality, we have
\begin{equation*}
        \mathrm{Pr} \left(X \le \mu - 1.5\sigma \right) 
    \le \mathrm{Pr} \left(|X - \mu | \ge 1.5\sigma \right) 
    \le \frac{1}{1.5^2} 
    < 0.49,
\end{equation*}
which is a contradiction.

\textbf{Step 2 (Partitioning into Regions):}  
Define $\mu_i \coloneqq \EE\left[X \cdot \mathbf{1}(X \in R_i)\right]$, with the regions $R_i$ defined in~\eqref{eq: general R_i def} and~\eqref{eq: general m_i}.
By the linearity of expectation, we have
\begin{equation}
    \sum_i \mu_i 
    = \EE \left[X \cdot  \sum_i \mathbf{1}(X \in R_i)\right] 
    = \EE\left[X \cdot \mathbf{1}\left(X \in \bigcup_i R_i\right)\right]
    = \EE[X] .
\end{equation}
   Therefore, it is sufficient to estimate each $\mu_i$.

\textbf{Step 3 (Ignoring Insignificant Regions):}
   For $i \ge 1$, we have $\max(R_i) \le m_i \sigma$ and $\min(R_i) \ge m_{i-1} \sigma$, where $m_i$ is as defined in~\eqref{eq: general m_i}. 
   Using $\max(R_i) \le m_i \sigma$, we have
   \begin{equation}
   \label{eq: mu_i bound}
       \mu_i 
        = \EE\left[X \cdot \mathbf{1}(X \in R_i)\right] 
        \le \max(X \in R_i) \cdot \EE\left[\mathbf{1}(X \in R_i)\right] 
        \le  m_i \sigma \cdot \PP\left[ X \in R_i \right]. 
   \end{equation}
   We now bound $\PP\left[ X \in R_i \right] $. First, recall that $\mu \le 3\sigma$ by our ``centering'' step in Step 1. Using this and $\min(X \in R_i) \ge  m_{i-1} \sigma$, we have
   \begin{equation}
   \label{eq: bounding P(X in R_i) part 1}
   \begin{aligned}
               \PP\left[ X \in R_i \right]
        \le \PP\left[ X \ge \min(X \in R_i)  \right] 
        \le \PP\left[ X \ge m_{i-1}  \sigma \right] 
        \le \PP\left[ X - \mu \ge \left(m_{i-1}  - 3\right) \sigma \right] .
   \end{aligned}
   \end{equation}
    For $i \ge 5$, using~\eqref{eq: bounding P(X in R_i) part 1}, Chebyshev's inequality, and the definition of $m_i$ (see~\eqref{eq: general m_i}) gives
    \begin{equation}
    \label{eq: bounding P(X in R_i) part 2}
        \PP\left[ X \in R_i \right] \le 
        \PP\left[ X - \mu \ge \left(m_{i-1} - 3 \right)\sigma \right] \le
        \PP\left[ |X - \mu| \ge \left(m_{i-1}  - 3\right)\sigma \right] \le \frac{1}{(m_{i-1} - 3)^2} =  \frac{4}{m_{i}^2}.
    \end{equation}
    Combining~\eqref{eq: mu_i bound} and~\eqref{eq: bounding P(X in R_i) part 2}, we have for $i \ge 5$ that
    \begin{equation}
        0 \le \mu_i \le 
        \frac{4 m_i}{m_i^2}  \sigma
        = 4 \sigma m_i^{-1}.
    \end{equation}    
    By a symmetric argument, we have an analogous bound for $i \le -5$. Combining these, we have 
    \begin{equation}
    \label{eq: |mu_i| upper bound}
        |\mu_i| \le 4 \sigma m_i^{-1} \quad  \text{for } |i| \ge 5.
    \end{equation}
   Consider the ``tail sum'' $ \sum_{i: |i| > i_{\max} } \mu_i$, where
   \begin{equation}
   \label{eq: i_max}
       i_{\max} = 
       \min_{i \ge 5} \left\{ i : 2^{-i} \le \frac{5\epsilon}{128\sigma} \right\} =
       \Theta\left( \log \left(\frac{\sigma}{\epsilon} \right) \right).
   \end{equation}
   Note that since $\frac{5}{8} \cdot 2^i \le \frac{5}{8} \cdot 2^i + 6 \le m_i \le 2^i$ (which can be verified using~\eqref{eq: general m_i} and induction), we have
   \begin{equation}
   \label{eq: sum of m_i}
    m_i^{-1} \le  \frac{8}{5} \cdot 2^{-i} 
    \quad \text{and} \quad
       \sum_{i=1}^{i_{\max}} m_i =
       \Theta( 2^{i_{\max}} ) =
       \Theta( m_{i_{\max}} ) =
       O\left( \frac{\sigma}{\epsilon} \right).
   \end{equation}
   Using triangle inequality and~\eqref{eq: |mu_i| upper bound}--\eqref{eq: sum of m_i}, the tail sum can be bounded by  
   \begin{align*}
        \left| \sum_{i: |i| > i_{\text{max}}} \mu_i  \right| 
       \le \sum_{i < -i_{\text{max}}} |\mu_i| + \sum_{i > i_{\text{max}}}   |\mu_i| 
       \le 8 \sigma \sum_{i > i_{\text{max}}} m_i^{-1}
       \le \frac{64 \sigma }{5} \cdot \sum_{i > i_{\text{max}}} 2^{-i}
       = \frac{64 \sigma }{5} \cdot  2^{-i_{\text{max}}} 
       \le \frac{\epsilon}{2}.
   \end{align*}
   It follows that 
   \begin{equation}
      \left|  \EE[X]  -  \sum_{i: |i| \le i_{\max} } \mu_i \right|=
      \left| \sum_{i} \mu_i - \sum_{i: |i| \le i_{\text{max}}} \mu_i  \right|=
      \left| \sum_{i: |i| > i_{\text{max}}} \mu_i  \right| \le
      \frac{\epsilon}{2},
   \end{equation}
   and so it is sufficient to estimate $\mu_i$ for $|i| < i_{\text{max}}$; the rest can be estimated as being 0 while only contributing at most $\epsilon/2$ to the error.

\textbf{Step 4 (Studying Region-Wise Randomized Interval Queries):}
    For each $i$, write~$R_i = [a_i, b_i)$ and let~$T_i \sim \mathrm{Unif}(a_i, b_i)$.  Using the law of total expectation, we have
    \begin{equation}
    \label{eq: Pr(X in [a_i, T_i] step 1}
        p_{a_i} \coloneqq \mathrm{Pr}\left(X \in [a_i, T_i] \right) 
            = \EE\left[ \mathbf{1}\left(X \in [a_i, T_i] \right) \right] 
            = \EE \left[ \EE\left[ \mathbf{1}\left(X \in [a_i, T_i] \right) \mid X \right] \right] 
            = \EE \left[ \mathrm{Pr}\left(X \in [a_i, T_i]  \mid X \right) \right].
    \end{equation}
    Using the CDF of the uniform distribution $T_i \sim \mathrm{Unif}(a_i, b_i)$, we have
    \begin{equation*}
        \mathrm{Pr}\left(X \in [a_i, T_i]  \mid X = x \right) =
        \begin{cases}
           \mathrm{Pr}(T_i \ge x) = \dfrac{b_i - x}{b_i -a_i} & \text{if } x \in [a_i, b_i) \\ \\
            0 & \text{otherwise}
        \end{cases},
    \end{equation*}
    which can be rewritten as
    \begin{equation}
     \label{eq: Pr(X in [a_i, T_i | X]}
        \mathrm{Pr}\left(X \in [a_i, T_i]  \mid X \right)=
        \frac{(b_i - X) \cdot \mathbf{1}\left(X \in R_i \right)}{b_i - a_i}.
    \end{equation}
    Combining~\eqref{eq: Pr(X in [a_i, T_i] step 1}--\eqref{eq: Pr(X in [a_i, T_i | X]} gives
    \begin{equation}
    \label{eq: Pr(X in [a_i, T_i]}
        p_{a_i}  =
        \EE \left[ \frac{ \left(b_i - X\right) \cdot \mathbf{1}\left(X \in R_i \right)}{b_i - a_i} \right].
    \end{equation}
    Likewise, similar steps give
    \begin{equation}
    \label{eq: Pr(X in [T_i, b_i] step 1}
        p_{b_i} \coloneqq \mathrm{Pr}\left(X \in [T_i, b_i] \right) 
            =
        \EE \left[ \frac{ \left(X - a_i\right) \cdot \mathbf{1}\left(X \in R_i \right)}{b_i - a_i} \right].
    \end{equation}
     Using~\eqref{eq: Pr(X in [a_i, T_i]} and~\eqref{eq: Pr(X in [T_i, b_i] step 1}, linearity of expectation, and basic algebraic manipulations,
    we can verify that
    \begin{equation}
         a_i \cdot p_{a_i} + b_i \cdot p_{b_i}
         = \EE\left[X \cdot \mathbf{1}(X \in R_i) \right]  = \mu_i. 
    \end{equation}
    It follows that, to estimate $\mu_i$, it is sufficient to estimate $p_{a_i}$ and $p_{b_i}$.  We denote the estimates as $\hat{p}_{a_i}$ and $\hat{p}_{b_i}$ respectively, and we form them using empirical averages of (randomized) interval queries in the next step.

\textbf{Step 5 (Estimating $p_{a_i}$ and $p_{b_i})$:}
    Using the identity
    $
    p_{a_i} = \EE\left[ \mathbf{1}\left(X \in [a_i, T_i] \right) \right]
    $ 
    in~\eqref{eq: Pr(X in [a_i, T_i] step 1}, the learner can form an estimate $\hat{p}_{a_i}$ of $p_{a_i}$ as follows:
    \begin{enumerate}
        \item Generate random variables $T_{i,j} \sim \mathrm{Unif}(a_i, b_i)$ for $j = 1, \dots, n_i$
        for some $n_i$ that will be determined later;

        \item Ask the agent $n_i$ randomized interval queries 
        ``Is $X_{i,j} \in [a_i, T_{i,j}]?$";

        \item Compute the empirical averages based on the 1-bit feedback.
    \end{enumerate}
    The learner can also form an estimate $\hat{p}_{b_i}$ of $p_{b_i}$ using a similar procedure but with queries ``Is $X_{i,j} \in [T_{i,j}, b_i]?$''. We summarize the estimates as follows:
    \begin{equation}
    \label{eq: estimates p_a and p_b}
        \hat{p}_{a_i} = \frac{1} {n_i}
        \sum_{j=1}^{n_i}  \mathbf{1}\left(X_{i,j} \in [a_i, T_{i,j}] \right)
        \quad \text{and} \quad
        \hat{p}_{b_i} = \frac{1} {n_i}
        \sum_{j=n_i+1}^{2n_i}  \mathbf{1}\left(X_{i,j} \in [ T_{i,j}, b_i] \right).
    \end{equation}
     The number of samples used to form each pair $(\hat{p}_{a_i} , \hat{p}_{b_i} )$ is $2n_i$, and the procedure to obtain all pairs $\{(\hat{p}_{a_i} , \hat{p}_{b_i} )\}_i$ can be done in a non-adaptive manner. Observe that if the estimates $\hat{p}_{a_i}$ of $p_{a_i}$ and $\hat{p}_{b_i}$ of $p_{b_i}$ satisfy
    \begin{equation}
    \label{eq: good estimates of p_a and p_b}
        |\hat{p}_{a_i} - p_{a_i}| \le \frac{\epsilon}{ \left( 2 \cdot i_{\text{max}}  \right) \cdot \left( |a_i|+ |b_i|\right)}
        \quad \text{and} \quad
        |\hat{p}_{b_i} - p_{b_i}| \le \frac{\epsilon}{ \left( 2 \cdot i_{\text{max}}  \right) \cdot \left( |a_i|+ |b_i|\right)} ,
    \end{equation}
    then we have
    \begin{equation}
        \begin{aligned}
         | \mu_i - ( a_i \hat{p}_{a_i} + b_i  \hat{p}_{b_i} )|
         &= 
         | ( a_i p_{a_i} + b_i  p_{b_i} ) - ( a_i \hat{p}_{a_i} + b_i  \hat{p}_{b_i} )| 
         = 
         \left| a_i \cdot ( p_{a_i} -  \hat{p}_{a_i} ) + b_i \cdot ( {p}_{b_i} - \hat{p}_{b_i} ) \right| 
         \le \frac{\epsilon}{2 \cdot i_{\text{max}} },
    \end{aligned} 
    \end{equation}
    from which it follows that
    \begin{equation}
    \begin{aligned}
        \left|   \sum_{i: |i| \le i_{\max} } \mu_i  -
         \sum_{i: |i| \le i_{\max}} \left( a_i \hat{p}_{a_i} + b_i  \hat{p}_{b_i} \right)
         \right| 
         \le 
          \sum_{i: |i| \le i_{\max} }  \left|   \left( \mu_i  -
         \left( a_i \hat{p}_{a_i} + b_i  \hat{p}_{b_i} \right) \right)
         \right| 
         \le 
         \frac{\epsilon}{2}.
    \end{aligned} 
    \end{equation}
    Towards establishing~\eqref{eq: good estimates of p_a and p_b}, we set 
    \begin{equation}
    \label{eq: delta_i and eps_i}
         \delta_i = \frac{\delta}{4 \cdot i_{\text{max}}} = \frac{\delta}{\Theta\left( \log(\sigma/\eps)\right)} 
    \quad \text{and} \quad
        \epsilon_i \coloneqq 
        \frac{\epsilon}{ \left( 2 \cdot i_{\text{max}}  \right) \cdot \left( |a_i|+ |b_i|\right)} =
        \frac{\epsilon}{\Theta\left( \log(\sigma/\eps) \cdot 2^i \cdot \sigma\right)},
    \end{equation}
    where we recall $i_{\text{max}}$ from~\eqref{eq: i_max} as well as $\{|a_i|, |b_i| \} =\{ m_{i-1}, m_i \}$ and $m_i = \Theta(2^i)$ from~\eqref{eq: general R_i def} and~\eqref{eq: general m_i}.
    
    For $|i| \le 4$, we take
    \begin{equation}
    \label{eq: n_i for i <= 5}
        n_i = \left\lceil \frac{1}{2 \epsilon_i^2} \log \left( \frac{2}{\delta_i} \right)\right\rceil.
    \end{equation}
    Recalling $p_{a_i}$ and $\hat{p}_{a_i}$ from~\eqref{eq: Pr(X in [a_i, T_i] step 1} and~\eqref{eq: estimates p_a and p_b}, applying Hoeffding's inequality for each $|i| \le 4$ gives:
    \begin{equation}
    \label{eq: start of Hoeffding}
        \mathrm{Pr}\left( |p_{a_i} -  \hat{p}_{a_i}| > \epsilon_i \right)
        =  \mathrm{Pr}\left(
        \left|
        \EE\left[ \mathbf{1}\left(X \in [a_i, T_i] \right) \right]  - \frac{1} {n_i}
        \sum_{j=1}^{n_i}  \mathbf{1}\left(X_{i,j} \in [a_i, T_{i,j}] \right)
        \right| >  \epsilon_i 
        \right) 
        \le  2 \exp\left( -2n_i \epsilon_i^2 \right) 
        \le \delta_i.
    \end{equation}
    For $|i| \ge 5$, we take
    \begin{equation}
    \label{eq: n_i for i >= 5}
        n_i 
        =\left\lceil
        \left( \frac{8}{m_i^2} \frac{1}{\epsilon_i^2}  + \frac{2}{3} \frac{1}{\epsilon_i} \right) \cdot \log\left(\frac{2}{\delta_i} \right)
        \right\rceil  
        \ge\left\lceil
        \left( \frac{2 \, \mathrm{Pr}(X \in R_i)}{ \epsilon_i^2}  + \frac{2}{3} \frac{1}{\epsilon_i} \right) \cdot \log\left(\frac{2}{\delta_i} \right)
        \right\rceil,
    \end{equation}
    where the inequality follows from~\eqref{eq: bounding P(X in R_i) part 2}. Applying Bernstein's inequality~\cite{vershynin2026high}[Theorem 2.9.5] to the i.i.d. mean zero bounded random variables 
    \[
        Y_{i,j} \coloneqq \mathbf{1}\left(X_{i,j} \in [a_i, T_{i,j}] \right) -  \EE\left[ \mathbf{1}\left(X \in [a_i, T_i] \right) \right],
    \]
    we obtain:
    \begin{align}
         \mathrm{Pr}\left( |p_{a_i} -  \hat{p}_{a_i}| > \epsilon_i \right)
        &=  \mathrm{Pr}\left(
        \left|
        \EE\left[ \mathbf{1}\left(X \in [a_i, T_i] \right) \right]  - \frac{1} {n_i}
        \sum_{j=1}^{n_i}  \mathbf{1}\left(X_{i,j} \in [a_i, T_{i,j}] \right)
        \right| >  \epsilon_i 
        \right) 
        \label{eq: start of Bernstein} \\
        &=  
        \mathrm{Pr}\left(
        \left|
        \frac{1} {n_i}
        \sum_{j=1}^{n_i}  Y_{i,j}
        \right| >  \epsilon_i 
        \right) \\
         & \le 2 \exp\left(-\frac{n_i \epsilon_i^2}{2 \,\EE \left[Y_{ij}^2 \right] + \frac{2}{3} \epsilon_i }\right) \\
        &= 2 \exp\left(-\frac{n_i \epsilon_i^2}{2 \,\mathrm{Var}\left(\mathbf{1}\left(X \in [a_i, T_i] \right) \right) + \frac{2}{3} \epsilon_i }\right) 
        \label{eq: E[Y_{ij}^2] expand}\\
         &\le 2 \exp\left(-\frac{n_i  }{ 2 \, \mathrm{Pr}(X \in R_i) \frac{1}{\epsilon_i^2} + \frac{2}{3}  \frac{1}{\epsilon_i} }\right)
         \label{eq:apply_var_ub} \\
        & \le \delta_i
         \label{eq: end of Bernstein},
    \end{align}
    where in \eqref{eq: E[Y_{ij}^2] expand} we use $X_{i,j} \stackrel d= X $ and $T_{i,j} \stackrel d= T_i$ to derive
    \[
    \EE\left[Y_{ij}^2 \right] = 
    \EE\left[ \left( \mathbf{1}\left(X_{i,j} \in [a_i, T_{i,j}] \right) -  \EE\left[ \mathbf{1}\left(X \in [a_i, T_i] \right) \right] \right)^2 \right]
    =
    \mathrm{Var}\left(\mathbf{1}\left(X \in [a_i, T_i] \right) \right),
    \]
    and in \eqref{eq:apply_var_ub} we use
    $\mathrm{Var}(\mathrm{Ber}(p)) = p(1-p) \le p$
    and $T_i \sim \mathrm{Unif}(a_i, b_i)$ to derive
    \[
         \mathrm{Var}\left(\mathbf{1}\left(X \in [a_i, T_i] \right) \right) \le
         \mathrm{Pr}\left(X \in [a_i, T_i] \right) 
         \leq   \mathrm{Pr}\left(X \in [a_i, b_i] \right) 
         = \mathrm{Pr}(X \in R_i).
    \]
    Likewise, we have $\mathrm{Pr}\left( |p_{b_i} -  \hat{p}_{b_i}| > \epsilon_i \right) \le \delta_i$. 
    
    We now substitute $\delta_i$ and $\epsilon_i$ from~\eqref{eq: delta_i and eps_i} into $n_i$.  
    For $|i| \le 4$, substituting these into $n_i$ from~\eqref{eq: n_i for i <= 5} gives
    \begin{equation}
    \label{eq: n_i for i<=4}
        n_i 
       = O \left( \underbrace{2^{2i}}_{\text{ bounded by $2^{8}$}} \cdot  \frac{\sigma^2}{\epsilon^2}
        \log^2\left( \frac{\sigma}{\epsilon}\right)
        \log\left( \frac{\log\left( \frac{\sigma}{\epsilon}\right)}{\delta}\right)\right)
        = O \left(\frac{\sigma^2}{\epsilon^2}
        \log^2\left( \frac{\sigma}{\epsilon}\right)
        \log\left( \frac{\log\left( \frac{\sigma}{\epsilon}\right)}{\delta}\right)\right).
    \end{equation}
    For $4 \le |i| \le i_{\max}$, substituting into $n_i$ from~\eqref{eq: n_i for i >= 5} and recalling $m_i = \Theta(2^i)$ from \eqref{eq: general m_i} gives
    \begin{equation}
    \label{eq: n_i for i>=5}
    \begin{aligned}
         n_i 
        &= O \left( \left(  \frac{1}{2^{2i}} \frac{2^{2i} \sigma^2}{\epsilon^2} \cdot \log^2\left( \frac{\sigma}{\epsilon}\right)   + \frac{2^{i} \sigma}{\epsilon} \cdot \log\left( \frac{\sigma}{\epsilon}\right) \right) \cdot \log\left( \frac{ \log\left( \frac{\sigma}{\epsilon}\right) }{\delta}\right)\right) \\
        &= O \left( \left(  \frac{\sigma^2}{\epsilon^2} \cdot \log\left( \frac{\sigma}{\epsilon}\right)   + \frac{2^{i} \sigma}{\epsilon} \right) \cdot
        \log\left( \frac{\sigma}{\epsilon}\right) \cdot \log\left( \frac{ \log\left( \frac{\sigma}{\epsilon}\right) }{\delta}\right)\right).
    \end{aligned}
    \end{equation}
    Summing up all $n_i$, we obtain
    \begin{equation}
    \label{eq: sum of n_i}
    \begin{aligned}
        \sum_{i: |i| \le i_{\max}} n_i 
         &= 2\left( \sum_{1 \le i \le  4} n_i +
         \sum_{5 \le i \le  i_{\max}} n_i \right) \\
         &=
        O \left(  \left( i_{\max}  \frac{\sigma^2}{\epsilon^2} \cdot \log\left( \frac{\sigma}{\epsilon}\right)   + \sum_{i \le  i_{\max}} \frac{2^{i} \sigma}{\epsilon} \right) \cdot
        \log\left( \frac{\sigma}{\epsilon}\right) \cdot \log\left( \frac{ \log\left( \frac{\sigma}{\epsilon}\right) }{\delta}\right)\right)  \\
          &= 
          O \left(  \left(\frac{\sigma^2}{\epsilon^2} \cdot \log^2\left( \frac{\sigma}{\epsilon}\right)   + \frac{2^{i_{\max}} \sigma}{\epsilon} \right) \cdot
        \log\left( \frac{\sigma}{\epsilon}\right) \cdot \log\left( \frac{ \log\left( \frac{\sigma}{\epsilon}\right) }{\delta}\right)\right)  \\
         &= 
          O \left(  \left(\frac{\sigma^2}{\epsilon^2} \cdot \log^2\left( \frac{\sigma}{\epsilon}\right)   + \frac{\sigma^2}{\epsilon^2} \right) \cdot
        \log\left( \frac{\sigma}{\epsilon}\right) \cdot \log\left( \frac{ \log\left( \frac{\sigma}{\epsilon}\right) }{\delta}\right)\right) \\
        &= 
          O \left(  \frac{\sigma^2}{\epsilon^2} \cdot  \log^3\left( \frac{\sigma}{\epsilon}\right)  \cdot \log\left( \frac{ \log\left( \frac{\sigma}{\epsilon}\right) }{\delta}\right)\right),
    \end{aligned}
    \end{equation}
    where the second last step follows since $2^{i_{\max}} = O\big( \frac{\sigma}{\epsilon} \big)$ (see~\eqref{eq: sum of m_i}).

\section{Equivalence of Randomized Interval Queries and Stochastic Rounding in Step 4}
\label{appendix: SQ}

    In this appendix, we show that our randomized interval queries from Step 4 of Section \ref{sec: algorithm} can be interpreted as performing a form of binary stochastic quantization.  Note that this connection is presented purely for the sake of intuition, and it is not needed in the proof of Theorem \ref{thm: main}.

    For each $i$, let $R_i = [a_i, b_i)$ as before, and define the stochastic quantizer $\mathrm{SQ}_i(\cdot)$ as follows:
    \begin{equation}
    \label{eq: SQ_i}
        \mathrm{SQ}_i(x) =
        \begin{cases}
            0 & \text{if } x \not\in R_i \\
            a_i & \text{with probability} \frac{b_i - x}{b_i - a_i}  \text{ if } x \in R_i\\
            b_i & \text{with probability} \frac{x - a_i}{b_i - a_i} \text{ if } x \in R_i.
        \end{cases}
    \end{equation}
   As before, we write $p_{a_i} \coloneqq \mathrm{Pr}\left(X \in [a_i, T_i] \right)$
   and  $p_{b_i} \coloneqq \mathrm{Pr}\left(X \in [T_i, b_i] \right)$.
    We now show that $p_{a_i}$ (resp. $p_{b_i}$)
    is equivalent to the probability of $X$ being in $R_i$ and getting rounded down to $a_i$ (resp. rounded up to $b_i$) by $\mathrm{SQ}_i$, i.e.,
    \[
    p_{a_i} =  \mathrm{Pr}(X \in R_i \cap \mathrm{SQ}_i(X) = a_i)
    \quad \text{and} \quad
    p_{b_i} =  \mathrm{Pr}(X \in R_i \cap \mathrm{SQ}_i(X) = b_i).
    \]
    Using~\eqref{eq: SQ_i} as well as standard properties of conditional probability, indicator functions, Bernoulli random variables, and linearity of expectation, we have
    \begin{equation}
    \label{eq: Pr(X in R_i and SQ_i(X) = a_i)}
    \begin{aligned}
        \mathrm{Pr}(X \in R_i \cap \mathrm{SQ}_i(X) = a_i)
        &= \mathrm{Pr}(X \in R_i)  \cdot 
        \mathrm{Pr}\left(\mathrm{SQ}_i(X) = a_i  \mid X \in R_i \right) \\
        &= \mathrm{Pr}(X \in R_i)  \cdot 
        \EE \left[\mathbf{1}\left(\mathrm{SQ}_i(X) = a_i  \right) \mid X \in R_i \right] \\
        &= \mathrm{Pr}(X \in R_i)  \cdot \EE\left[\frac{b_i - X}{b_i - a_i} \ \Bigg\vert\  X \in R_i \right]  \\
        &= 
        \frac{b_i \cdot \mathrm{Pr}(X \in R_i) - \EE\left[X \mid  X \in R_i \right] \cdot \mathrm{Pr}(X \in R_i)}{b_i - a_i}.
    \end{aligned}
    \end{equation}
    Moreover, using $p_{a_i} =  \EE \left[  \left(b_i - X\right) /(b_i - a_i) \cdot \mathbf{1}\left(X \in R_i \right)  \right]$ (see~\eqref{eq: Pr(X in [a_i, T_i]}) and~\eqref{eq: Pr(X in R_i and SQ_i(X) = a_i)} as well as linearity of expectation and law of total expectation, we have
    \begin{equation}
    \begin{aligned}
        p_{a_i}
            &=
        \EE \left[ \frac{ \left(b_i - X\right) \cdot \mathbf{1}\left(X \in R_i \right)}{b_i - a_i} \right] \\
         &=
        \frac{ b_i \cdot \EE \left[  \mathbf{1}\left(X \in R_i \right)\right] - \EE \left[   X \cdot \mathbf{1}\left(X \in R_i \right)  \right]}{b_i - a_i} \\
        &= \frac{b_i \cdot \mathrm{Pr}(X \in R_i) - \EE\left[X \mid  X \in R_i \right] \cdot \mathrm{Pr}(X \in R_i)}{b_i - a_i} \\
        &= \mathrm{Pr}(X \in R_i \cap \mathrm{SQ}_i(X) = a_i)
    \end{aligned}
    \end{equation}
    as desired. Analogous steps give $p_{b_i} =  \mathrm{Pr}(X \in R_i \cap \mathrm{SQ}_i(X) = b_i)$.

\section{Lower Bound and Adaptivity Gap}
\label{appendix: Lower Bounds}
\subsection{Proof of Theorem~\ref{thm: adaptive lower bound} (General Lower Bound)}
\label{appendix: Lower Bound for Adaptive Queries}

    Even if the $(\epsilon, \delta)$-PAC estimator has no 1-bit constraint, the lower bound
    $
    n = \Omega \big( \frac{\sigma^2}{\epsilon^2}  \log\left(\frac{1}{\delta} \right) \big)
    $ is well known.  
    For instance, this can be derived via a reduction to distinguishing two Bernoulli distributions~\cite[Section 4]{LeeCSCI1951}.
    Therefore, it is sufficient for us to establish that $n = \Omega\left(\log \frac{\lambda}{\sigma}\right)$.

    We create $N = \Theta(\lambda/\sigma)$ instances of ``hard-to-distinguish'' distribution pairs, which we will reuse in the proof of Theorem~\ref{thm: non-adaptive lower bound} in Appendix~\ref{appendix: Lower Bound for Non-adaptive Queries}.
    Divide $[-\lambda, \lambda]$ into a grid of $N  = \lambda / \sigma - 1 $  ``center-points'' spaced $2\sigma$ apart,\footnote{For convenience, we assume that $\lambda$ is an integer multiple of $2\sigma$.  This is justified by a simple rounding argument and the fact that when $\lambda = \Theta(\sigma)$ the $\Omega\big(\log\frac{\lambda}{\sigma} \big)$ lower bound is trivial.  } i.e., the center-points are
\begin{equation}
\label{eq: c_j spacing}
    c_j = - \lambda + 2j \sigma \quad
    \text{for each }j = 1, 2 \dots, N.
\end{equation}
For each instance $j$, we define two probability distributions
$D_{j, -}$ and $D_{j, +}$, each with a two-point support set $\{c_j-\sigma/2, c_j+\sigma/2\}$, as follows:
\begin{equation}
\label{eq: D_j+ and D_j-}
    \begin{aligned}
         D_{j,-} \colon& \mathrm{Pr}\left(X = c_j + \frac{\sigma}{2} \right) = \frac{1}{2} - \frac{\epsilon}{\sigma} =
        1 - \mathrm{Pr}\left(X = c_j - \frac{\sigma}{2} \right)
        \implies
        \EE[X ] = c_j - \epsilon \\
    D_{j,+} \colon & \mathrm{Pr}\left(X = c_j + \frac{\sigma}{2} \right) = \frac{1}{2} + \frac{\epsilon}{\sigma} =
        1 - \mathrm{Pr}\left(X = c_j - \frac{\sigma}{2} \right)
        \implies
        \EE[X ] = c_j + \epsilon.
    \end{aligned}
\end{equation}
    We readily observe the following:
    \begin{itemize}
        \item By the assumption $\epsilon < \frac{\sigma}{2}$, each each of these $2N$ distributions has their mean in $[-\lambda, \lambda]$;
        \item Since a distribution on $[a,b]$ as variance at most $\frac{(b-a)^2}{4}$, each of these $2N$ distributions has variance at most $\sigma^2$.
    \end{itemize}
    Therefore, when the distributions are restricted to only these $2N$ distributions, the task of being able to form an $\epsilon$-good estimation of the true mean of each unknown underlying distribution is at least as hard as being able to distinguish the distributions from each other.\footnote{Strictly speaking this is true when the algorithm is required to attain accuracy \emph{strictly smaller} than $\epsilon$, rather than {\em smaller or equal}, but this distinction clearly has no impact on the final result stated using $O(\cdot)$ notation, and by ignoring it we can avoid cumbersome notation.}  We proceed to establish a lower bound for this goal of \emph{identification}, also known as \emph{multiple hypothesis testing}.
    
    Let $\Theta$ be a uniform random variable over the $2N$ distributions, which implies
    \begin{equation}
    \label{eq: entropy of uniform dist}
        H(\Theta) = \log(2N),
    \end{equation}
    where $H(X) \coloneqq -\sum_{x \in \mathcal{X}} p(x) \log p(x)$ is the entropy function. Fix an adaptive mean estimator that makes $n$ queries, and let $Y^n = (Y_1, \dots, Y_n)$ be the resulting binary responses.
    Using the chain rule for mutual information (see e.g.~\cite[Theorem 3.7]{Polyanskiy_Wu_2025}) and the fact that each query yields at most 1 bit of information, we have
    \begin{equation}
        \label{eq: trivial mutual info bound}
            I(\Theta; Y^n) 
         = \sum_{k=1}^n   I\big(\Theta; Y_k \mid Y^{k-1}\big)
        \le \sum_{k=1}^n  H \big(Y_k \mid Y^{k-1} \big) 
        \le \sum_{k=1}^n  H (Y_k) 
        \le \sum_{k=1}^n  1
        = n.
    \end{equation}
    Moreover, Fano’s inequality (see~\cite[Theorem 3.12]{Polyanskiy_Wu_2025}) gives:
    \begin{equation}
        \label{eq: Fano inequality}
        H(\Theta \mid Y^n) \leq H_2(\delta) + \delta \log (2N-1)
      \le 1 + \delta \log(2N),
    \end{equation}
    where \(\delta\) is the error probability and $H_2(p) = -p \log_p - (1-p) \log(1-p)$ is the binary entropy function. 
    Using~\eqref{eq: entropy of uniform dist}--\eqref{eq: Fano inequality} and the definition of mutual information, we obtain
    \begin{equation}
     n 
     \ge I(\Theta; Y^n) 
     = H(\Theta) - H(\Theta \mid Y^n) 
     \ge  \log(2N) - 1 - \delta \log (2N)
     = (1-\delta) \log(2N) - 1.
    \end{equation}
    Combining this with $N = \Theta(\lambda/\sigma)$, we have
    \[
    n
    = \Omega( \left(1-\delta \right) \log N )
    = \Omega\left(\log \frac{\lambda}{\sigma}\right)
    \]
    as desired.

\subsection{Proof of Theorem~\ref{thm: non-adaptive lower bound} (Adaptivity Gap)}
\label{appendix: Lower Bound for Non-adaptive Queries}

    We consider the same instance as that of Section \ref{appendix: Lower Bound for Adaptive Queries}, and accordingly re-use the notation therein.  
    Before proving Theorem~\ref{thm: non-adaptive lower bound}, we first introduce the idea of an interval query being ``informative'' or ``uninformative'' for distinguishing between the distributions $D_{j,-}$ and $D_{j,+}$. 
    \begin{definition}[Informative Interval Queries] \label{def:informative}
       For a fixed interval query $Q =  ``\text{Is } X \in [a, b]?"$, we say that $Q$ is informative for the $j$-th pair of distributions $(D_{j,-},D_{j,+})$ if its binary feedback $B = \mathbf{1}\left\{X \in [a, b] \right\}$ satisfies
        \[
        \mathrm{Pr}_{X \sim  D_{j,-}}(B = 1) \ne \mathrm{Pr}_{X \sim  D_{j,+}}(B =1).
        \]
        Otherwise, $Q$ is said to be uninformative.
    \end{definition}
    The following lemma shows that each interval query can be simultaneously informative for at most two different pairs. 
    \begin{lemma}
    \label{lem: informative for at most 2 pairs}
        An interval query $Q =  ``\text{Is } X \in [a, b]?"$can be simultaneously informative for at most two different  $(D_{j,-},D_{j,+})$ pairs, i.e., at most two different values of $j$. 
    \end{lemma}
    \begin{proof}[Proof of Lemma~\ref{lem: informative for at most 2 pairs}]
        The claim follows from the following two facts:
    \begin{enumerate}
        \item For a fixed distribution pair (indexed by $j$), an interval query $Q = $ ``$\text{Is } X \in [a, b]?$'' is informative for distinguishing between $D_{j,-}$ and $D_{j,+}$ only if $[a, b]$ contains exactly one of the two support points $\{c_j \pm \sigma/2 \}$, i.e., $\big| [a, b] \cap \{c_j \pm \sigma/2 \} \big| = 1$.

        \item 
        There are at most two indices $j$ for which $\big| [a, b] \cap \{c_j \pm \sigma/2 \} \big| = 1$.
    \end{enumerate}
    Fact 1 can be verified by analyzing the binary feedback $B = \mathbf{1}\left\{X \in [a, b] \right\}$ for all cases of $[a, b] \cap \{c_j \pm \sigma/2 \}$: 
    \begin{equation*}
         \big|
        \left[a, b \right] \cap \{c_j \pm \sigma/2\} \big|  \in\{0,2\}
        \implies 
        \mathrm{Pr}_{X \sim  D_{j,-}}(B = 1) = \mathrm{Pr}_{X \sim  D_{j,+}}(B =1 ) 
        \implies 
        Q \text{ is uninformative},
    \end{equation*}
    and 
    \begin{equation}
    \label{eq: bernoulli p+ and p-}
        \big|
        \left[a, b \right] \cap \{c_j \pm \sigma/2\} \big|  = 1
        \implies 
        \big| 
        \mathrm{Pr}_{X \sim  D_{j,-}}(B = 1) - \mathrm{Pr}_{X \sim  D_{j,+}}(B =1 ) 
        \big| = \frac{2 \epsilon}{\sigma}
        \implies 
        Q \text{ is informative}.
    \end{equation}
    For Fact 2, we first observe from~\eqref{eq: c_j spacing} that the support points of all $2N$ distributions satisfy
    \[
         c_1- \frac{\sigma}{2} < c_1 + \frac{\sigma}{2} < c_{2} - \frac{\sigma}{2} <
         \cdots <  c_{N} - \frac{\sigma}{2} < c_{N} + \frac{\sigma}{2},
    \]
    with each pair $j$ having a unique disjoint interval $(c_j - \sigma/2, c_j + \sigma/2)$
    between its support points. An interval $[a, b]$ satisfies $\big| [a, b] \cap \{c_j \pm \sigma/2 \} \big| = 1$ if and only if exactly one endpoint of $[a, b]$ lies in the interval $(c_j - \sigma/2, c_j + \sigma/2)$. Since the gaps are disjoint and $[a, b]$ has only two endpoints, it follows that at most two indices $j$ satisfy $\big| [a, b] \cap \{c_j \pm \sigma/2 \} \big| = 1$.
    \end{proof}

 \begin{proof}[Proof of Theorem~\ref{thm: non-adaptive lower bound}]

   Consider an arbitrary algorithm that makes $n$ non-adaptive interval queries.  
   Recall the set of $2N$ distributions $\{ D_{j,-}, D_{j,+} \}_{j=1}^{N} \subseteq \mathcal{D(\lambda, \sigma)}$ constructed in the proof of Theorem~\ref{thm: adaptive lower bound}, where $N = \lambda/ \sigma -1$.  We will again establish a lower bound for this ``hard subset'' of distributions, but with different details to exploit the assumption of non-adaptive interval queries.
   
   Recall from Section \ref{appendix: Lower Bound for Adaptive Queries} that the means of the $2N$ distributions are pairwise separated by $2\epsilon$ or more, and thus, attaining $\epsilon$-accuracy implies being able to identify the underlying distribution from the hard subset.  We proceed to establish a lower bound for this goal of identification (multiple hypothesis testing).  

   Suppose that the true distribution is drawn uniformly at random from the $2N$ distributions in the hard subset.  By Yao’s minimax principle, the worst-case error probability is lower bounded by the average-case error probability of the best \emph{deterministic} strategy, so we may assume that the algorithm is deterministic (in the choice of queries and the procedure for forming the final estimate).
   
   Letting $(\hat{j},\hat{s})$ be the estimated index (in $\{1,\dotsc,N\}$) and sign (in $\{1,-1\}$), the average-case error probability is given by
   \begin{align}
    \mathrm{Pr}({\rm error}) 
        &= \frac{1}{2N} \sum_{j=1}^N \sum_{s \in \{+1,-1\}} \mathrm{Pr}_{j,s}( (\hat{j},\hat{s}) \ne (j,s)  ) \\
        &\ge \frac{1}{N}\sum_{j=1}^N \bigg( \underbrace{\frac{1}{2} \mathrm{Pr}_{j,+}\big( \hat{s} \ne 1\big) + \frac{1}{2}\mathrm{Pr}_{j,-}\big( \hat{s} \ne -1 \big)}_{=: \mathrm{Pr}_j({\rm error})} \bigg), \label{eq:def_Pr_j}
    \end{align}
    where $\mathrm{Pr}_{j,s}$ denotes probability when the underlying distribution is $D_{j,s}$.  

    For each $j = 1, \dots, N$, we define $n_j$ to be the algorithm's total number of interval queries that are informative (in the sense of Definition \ref{def:informative}) for distinguishing between $D_{j,-}$ and $D_{j,+}$.  Since the algorithm is deterministic and the $n$ queries are assumed to be non-adaptive (i.e., they must all be chosen in advance), it follows that the values $\{n_j\}_{j=1}^N$ are also deterministic.

    Recall from~\eqref{eq: bernoulli p+ and p-} that each informative query provides binary feedback that follows either $\mathrm{Bern}(p_{+})$ or $\mathrm{Bern}(p_{-})$, where $p_{+} = 1/2 + \epsilon/\sigma$ and $p_{-} =1/2  - \epsilon/\sigma = 1 - p_{+}$.
    Distinguishing between these two cases is a \emph{binary hypothesis testing} problem, and the associated error probability $\mathrm{Pr}_j(\text{error})$ is given by the $j$-th summand in \eqref{eq:def_Pr_j}.
    
    Using standard binary hypothesis testing lower bounds~\cite[Theorem 11.9]{LeeCSCI1951}, we have\footnote{We have re-arranged their result to express other quantities in term of the error probability.}
    \begin{equation}
    \label{eq: n_j hypothesis testing lower bound reexpressed}
        \mathrm{Pr}_j(\text{error}) >
        \exp\left( -c' \cdot n_j \cdot d_H^2(p_{+}, p_{-})  \right)
    \end{equation}
    for some constant $c'$, where $d_H^2(\mathbf{p}, \mathbf{q}) =  \frac{1}{2}\sum_{i} \left(\sqrt{p_i} - \sqrt{q_i} \right)^2$  
     is the Squared Hellinger distance.
     For $\mathrm{Bern}(p_{+})$ and $\mathrm{Bern}(p_{-})$, we have the following standard calculation:
     \begin{equation}
    \label{eq: d_H calculation bound}
        d_H^2(p_{+}, p_{-}) =
        \left(\sqrt{p_+} - \sqrt{p_-} \right)^2 =
        \left( \frac{p_+  - p_-}{\sqrt{p_+} + \sqrt{p_-} }\right)^2 =
        \frac{ |p_+  - p_-|^2}{  \left(  1 + 2\sqrt{p_+  p_-} \right)^2 } =
        \Theta\left( |p_{+} - p_{-}|^2\right) =
        \Theta\left(\frac{\eps^2}{\sigma^2} \right),
    \end{equation}
    where the equalities follow from the facts that $p_+ + p_- = 1$ and $p_+  p_- \in [0, 1/4]$.
    Combining~\eqref{eq: n_j hypothesis testing lower bound reexpressed} and~\eqref{eq: d_H calculation bound}, we obtain
    \begin{equation}
        \mathrm{Pr}_j(\text{error}) 
         >
        \exp\left( -c'' \cdot \frac{n_j \, \eps^2}{\sigma^2}  \right)
    \end{equation}
    for some constant $c'' > 0$. 
    Applying Jensen's inequality (since $\exp$ is convex) and using
    $\sum_{j=1}^N n_j \le 2n$ (see Lemma~\ref{lem: informative for at most 2 pairs}), it follows that
    \[
    \frac{1}{N} \sum_{j=1}^N \mathrm{Pr}_j(\text{error}) >
    \frac{1}{N} \sum_{j=1}^N \exp\left( -c'' \cdot \frac{n_j \, \eps^2}{\sigma^2}  \right)
    \ge
     \exp\left( -c'' \cdot \frac{\eps^2}{\sigma^2} \cdot  \frac{1}{N} \sum_{j=1}^N n_j \right)
     \ge
     \exp\left( -c'' \cdot \frac{\eps^2}{\sigma^2} \cdot  \frac{2n}{N}  \right).
    \]
    It follows that if 
    \[
    n 
    <
    \frac{1}{4c''} \cdot \frac{\lambda \sigma}{ \eps^2} \log\left(\frac{1}{\delta}\right) 
    =
     \frac{1}{4c''} \cdot \frac{\lambda}{\sigma} \cdot \frac{\sigma^2}{\eps^2} \cdot \log\left(\frac{1}{\delta}\right) 
     =
     \frac{1}{4c''} \cdot (N+1) \cdot \frac{\sigma^2}{\eps^2} \cdot \log\left(\frac{1}{\delta}\right) 
    \le 
    \frac{N}{2c''} \frac{\sigma^2}{\eps^2} \log\left(\frac{1}{\delta}\right),
    \]
    then the average error probability is lower bounded by
    \[
    \frac{1}{N} \sum_{j=1}^N \mathrm{Pr}_j(\text{error}) >
     \exp\left( -c'' \cdot \frac{\eps^2}{\sigma^2} \cdot  \frac{2n}{N}  \right)
     \ge 
     \exp \left( \log\left(\frac{1}{\delta}  \right) \right) = \delta.
    \]
    Therefore, to attain an error probability no higher than $\delta$, we must have
    \[
    n = \Omega\left( \frac{\lambda \sigma}{ \eps^2} \log\left(\frac{1}{\delta}\right)  \right)
    \]
    as desired.
 \end{proof}

\section{Improvements for Random Variables with Stronger Tail Decay}
\subsection{Proof of Theorem~\ref{thm: higher-order variant} (Improvement with Finite Higher-order Central Moments)}
\label{appendix: finite k}
The main difference compared to the case with only bounded variance is that we now have a better tail bound through the higher-moment Chebyshev's inequality:
\begin{equation}
\label{eq: higher moment Chebyshev}
    \mathrm{Pr}(|X - \mu| \ge t) \le 
    \frac{\EE | X - \mu|^k}{t^k} \le
    \frac{\sigma^k}{t^k}.
\end{equation}
Since the proof mostly follows that of Theorem \ref{thm: main}, we focus our attention on the steps that are different.

\textbf{Modified Step 2:} We let the width of the regions $R_i$ grow doubly exponentially instead of exponentially.
Specifically, we still let $R_i$ have the form in~\eqref{eq: general R_i def}, but we modify $m_i$ in~\eqref{eq: general m_i} as follows:
    \begin{equation}
    \label{eq: finite k m_i}
        m_i = 
        \begin{cases}
         0 & \text{if } i = 0
        \\ \\
        2^{\left(k/2\right)^{i}} & \text{if } 1 \le i \le 4
        \\ \\  (m_{i-1} -3 )^{k/2}  & \text{if }  i \ge 5
        \end{cases}
    \end{equation}
     Note that the last case can be expanded as $\underbrace{\left( \left( \left( 2^{(k/2)^4} - 3 \right)^{k/2} - 3 \right)^{k/2} \cdots - 3 \right)^{k/2}}_{i-4 \text{ times}}$, from which we can verify by induction that $m_i$ scales doubly exponentially according to $\Theta\big( 2^{ (k/2)^{i}} \big)$.

     \textbf{Modified Step 3:} Because $m_i = \Theta\big( 2^{ (k/2)^{i}} \big)$, we expect $i_{\max}$ to have $\log_{k/2} \log \left(\sigma/\epsilon \right)$ scaling instead of $\log \left(\sigma/\epsilon \right)$. We proceed to show this.  
    For $|i| \ge 5$, using steps similar to those in~\eqref{eq: mu_i bound}--\eqref{eq: |mu_i| upper bound}, but with higher-moment Chebyshev's inequality~\eqref{eq: higher moment Chebyshev} and the modified definition of $m_i$ gives 
    \begin{equation}
    \label{eq: higher moment bounding P(X in R_i) part 2}
        \PP\left[ X \in R_i \right] \le 
        \PP\left[ X - \mu \ge \left(m_{i-1} -3\right)\sigma \right] \le 
         \frac{1}{(m_{i-1}-3)^k} 
         = \frac{1}{m_i^2},
    \end{equation}
    which implies
    \begin{equation}
    \label{eq: higher moment mu_i bound}
        |\mu_i| \le \sigma m_i^{-1}
         \text{ for } |i| \ge 5.
    \end{equation}    
     Consider the ``tail sum'' $ \sum_{i: |i| > i_{\max} } \mu_i$, 
   where
   \begin{equation}
   \label{eq: higher moment i_max}
       i_{\max} = 
       \min \left\{ i : m_{i+1}^{-1} \le \frac{\epsilon}{8\sigma} \right\} =
       \min \left\{ i : m_{i+1} \ge \frac{8\sigma}{\epsilon} \right\} =
       \Theta\left( \log_{k/2} \log \left(\frac{\sigma}{\epsilon} \right) \right).
   \end{equation}
   Note that due to the ``super-geometric" growth of $m_i$, we have
   \begin{equation}
   \label{eq: higher moment sum of m_i}
    m_{i+1} \ge \frac{m_i}{2}
    \quad \text{and} \quad
       \sum_{i=1}^{i_{\max}} m_i = 
       \Theta( m_{i_{\max}} ) =
       O\left( \frac{\sigma}{\epsilon} \right).
   \end{equation}
   Using~\eqref{eq: higher moment mu_i bound}--\eqref{eq: higher moment sum of m_i}, the tail sum can be bounded by  
   \begin{align*}
       \left|  \sum_{i < -i_{\text{max}}} \mu_i + \sum_{i > i_{\text{max}}} \mu_i
       \right|
       \le \sum_{i < -i_{\text{max}}} |\mu_i| + \sum_{i > i_{\text{max}}}   |\mu_i| 
       \le 2\sigma \sum_{i > i_{\text{max}}}  m_i^{-1} 
       \le 2\sigma \left( \frac{\epsilon}{8 \sigma} +  \frac{\epsilon}{16 \sigma} +  \frac{\epsilon}{32 \sigma} + \cdots  \right)  
       \le \frac{\epsilon}{2}.
   \end{align*}
   It follows that 
   \begin{equation}
      \left|  \EE[X]  -  \sum_{i: |i| \le i_{\max} } \mu_i \right|=
      \left| \sum_{i} \mu_i - \sum_{i: |i| \le i_{\text{max}}} \mu_i  \right|=
      \left| \sum_{i: |i| > i_{\text{max}}} \mu_i  \right| \le
      \frac{\epsilon}{2},
   \end{equation}
   and so it is sufficient to estimate $\mu_i$ for $|i| < i_{\text{max}}$.

   \textbf{Modified Step 5:} We adjust $\delta_i$ and $\epsilon_i$ according to the new $m_i$ and $i_{\text{max}}$, which gives us a smaller $n_i$ and $\sum_{i: |i| \le i_{\max}} n_i$.
   Specficially, we set
   \begin{equation}
       \label{eq: finite k delta_i and eps_i}
       \delta_i = \frac{\delta}{4 \cdot i_{\text{max}}}
    = \frac{\delta}{\Theta\left(   \log_{k/2} \log \left(\frac{\sigma}{\epsilon} \right) \right)}
    \quad \text{and} \quad
        \epsilon_i \coloneqq 
        \frac{\epsilon}{ \left( 2 \cdot i_{\text{max}}  \right) \cdot \left( |a_i|+ |b_i|\right)} =
                 \frac{\epsilon}{ \Theta 
        \left( \log_{k/2} \log \left(\frac{\sigma}{\epsilon} \right)  \cdot m_i \cdot \sigma \right)} .
   \end{equation}
    For $|i| \le 4$, we take
    $
        n_i = \left\lceil \frac{1}{2 \epsilon_i^2} \log \left( \frac{2}{\delta_i} \right)\right\rceil,
   $
    and for $|i| \ge 5$, we take
    \begin{align}
        n_i 
        &=\left\lceil
        \left( \frac{2}{m_i^{2}} \frac{1}{\epsilon_i^2}  + \frac{2}{3} \frac{1}{\epsilon_i} \right) \cdot \log\left(\frac{2}{\delta_i} \right)
        \right\rceil  
        \ge\left\lceil
        \left(  \frac{2 \,\mathrm{Pr}(X \in R_i)}{ \epsilon_i^2}  + \frac{2}{3} \frac{1}{\epsilon_i} \right) \cdot \log\left(\frac{2}{\delta_i} \right)
        \right\rceil,
    \end{align}
    where the inequality follows from~\eqref{eq: higher moment bounding P(X in R_i) part 2}. 
    Applying Hoeffding's inequality for each $|i| \le 4$ as in~\eqref{eq: start of Hoeffding} and Bernstein's inequality for each $|i| \ge 5$  as in~\eqref{eq: start of Bernstein}--\eqref{eq: end of Bernstein}, we obtain:
    \begin{equation}
         \mathrm{Pr}\left( |p_{a_i} -  \hat{p}_{a_i}| > \epsilon_i \right)
         \le \delta_i
         \quad \text{and} \quad
         \mathrm{Pr}\left( |p_{b_i} -  \hat{p}_{b_i}| > \epsilon_i \right) \le \delta_i.
    \end{equation}
   To substitute $\delta_i$ and $\epsilon_i$ from~\eqref{eq: finite k delta_i and eps_i} into $n_i$, we use steps similar to~\eqref{eq: n_i for i<=4} and~\eqref{eq: n_i for i>=5}, which gives:
    \begin{equation}
    \begin{aligned}
        n_i 
        = O \left(\frac{\sigma^2}{\epsilon^2}
        \left( \log_{k/2} \log \left( \frac{\sigma}{\epsilon} \right)\right)^2
        \log\left( \frac{\log_{k/2} \log \left( \frac{\sigma}{\epsilon}\right)}{\delta}\right) 
        \right)
        \quad  \text{for } |i| \le 4
    \end{aligned}
    \end{equation}
    and
    \begin{equation}
    \begin{aligned}
         n_i 
        = O \left( \left(  \frac{ \sigma^2}{\epsilon^2} \cdot \log_{k/2} \log \left( \frac{\sigma}{\epsilon}\right)   + \frac{m_i \sigma}{\epsilon} \right) \cdot
        \log_{k/2} \log \left( \frac{\sigma}{\epsilon}\right) \cdot \log\left( \frac{ \log_{k/2} \log \left( \frac{\sigma}{\epsilon}\right) }{\delta}\right)\right)
        \quad  \text{for } 5 \le |i| \le i_{\max}.
    \end{aligned}
    \end{equation}
    Summing up all $n_i$ as in~\eqref{eq: sum of n_i}, we obtain
    \begin{equation}
    \begin{aligned}
        \sum_{i: |i| \le i_{\max}} n_i 
        &=O \left(  \left( i_{\max}  \frac{\sigma^2}{\epsilon^2} \cdot \log_{k/2} \log \left( \frac{\sigma}{\epsilon}\right)   + \sum_{i \le  i_{\max}} \frac{m_i \sigma}{\epsilon} \right) \cdot
        \log_{k/2} \log \left( \frac{\sigma}{\epsilon}\right) \cdot \log\left( \frac{ \log_{k/2} \log\left( \frac{\sigma}{\epsilon}\right) }{\delta}\right)\right)  \\
         &=
          O \left(  \left(\frac{\sigma^2}{\epsilon^2} \cdot  \left( \log_{k/2} \log \left( \frac{\sigma}{\epsilon} \right)\right)^2  
          +
          \frac{\sigma}{\epsilon}  \cdot  \frac{\sigma}{\epsilon} \right) \cdot
        \log_{k/2} \log \left( \frac{\sigma}{\epsilon}\right) \cdot \log\left( \frac{ \log_{k/2} \log \left( \frac{\sigma}{\epsilon}\right) }{\delta}\right)\right) \\
        &=
          O \left(  \frac{\sigma^2}{\epsilon^2} \cdot   \left( \log_{k/2} \log \left( \frac{\sigma}{\epsilon} \right)\right)^3   \cdot \log\left( \frac{ \log_{k/2} \log \left( \frac{\sigma}{\epsilon}\right) }{\delta}\right)\right),
    \end{aligned}
    \end{equation}
     where the second step follows from~\eqref{eq: higher moment sum of m_i}.

\subsection{Proof of Theorem~\ref{thm: subgaussian variant} (Improvement for Sub-Gaussian Random Variables)}
\label{appendix: subgaussian}
The main difference is that we now have an even faster tail decay through the sub-Gaussian tail bound~\eqref{eq: subgaussian bound}.

\textbf{Modified Step 2:} 
Due to the strong tail decay of sub-Gaussian random variables, we can let the width of regions $R_i$ grow much more rapidly.  Specifically, we keep $R_i$ as in~\eqref{eq: general R_i def} but modify $m_i$ in~\eqref{eq: general m_i} as follows:
    \begin{equation}
    \label{eq: subgaussian m_i}
        m_i = 
        \begin{cases}
         0 & \text{if } i = 0
        \\ \\
        \exp\left( \frac{m_{i-1}^2}{2} \right) & \text{if } 1 \le i \le 4
        \\ \\
         \exp\left(\frac{(m_{i-1}-3)^2}{4}\right) 
         & \text{if }  i \ge 5.
        \end{cases}
    \end{equation}
     Note that $m_i$ scales according to a tower of exponentials of height $i$, which can be verified by induction:
     \begin{equation}
     \label{eq: subgaussian m_i scaling}
          m_i =
         \Theta \left(
         \underbrace{\exp \left( \exp \left( \cdots \exp\left(  \Theta(1) \right) \right)  \right)}_{i\text{ times}}
         \right).
     \end{equation}
\textbf{Modified Step 3:} Because $m_i$ scales according to a tower of exponentials, we expect $i_{\max}$ to have $\log^* \left(\sigma/\epsilon \right)$ scaling.  
    Because the arguments are almost identical to those in modified Step 3 of Appendix~\ref{appendix: finite k} (improvement for random variables with finite $k$-th central moment), we will omit most of the details. The main difference is that we use the sub-Gaussian bound~\eqref{eq: subgaussian bound} and the modified definition of $m_i$ (see~\eqref{eq: subgaussian m_i}) in obtaining
    \begin{equation}
    \label{eq: subgaussian bounding P(X in R_i) part 2}
      \PP\left[ X \in R_i \right] \le 
        \PP\left[ X - \mu \ge \left(m_{i-1} -3\right)\sigma \right] \le 
        \exp\left(-\frac{(m_{i-1}-3)^2}{2} \right)
        =  \frac{1}{m_i^2}.
    \end{equation} 
   Consequently, we have
   \begin{equation}
   \label{eq: subgaussian i_max}
       i_{\max} =
       \min \left\{ i : \frac{\sigma}{m_{i+1}} \le \frac{\epsilon}{8} \right\} 
       = \Theta\left( \log^* \left(\frac{\sigma}{\epsilon} \right) \right)
       \quad \text{and} \quad 
        \sum_{i=1}^{i_{\max}} m_i = 
       \Theta( m_{i_{\max}} ) =
       O\left( \frac{\sigma}{\epsilon} \right).
   \end{equation}
   
  \textbf{Modified Step 5:} We adjust $\delta_i$ and $\epsilon_i$ according to the new $m_i$ and $i_{\text{max}}$, which gives us a smaller $n_i$ and $\sum_{i: |i| \le i_{\max}} n_i$. As the steps are almost identical to those in modified Step 5 of Appendix~\ref{appendix: finite k}, we will omit most of the details for brevity. 
  We set
  \begin{equation}
       \label{eq: sugbaussian delta_i and eps_i}
       \delta_i = \frac{\delta}{4 \cdot i_{\text{max}}}
    = \frac{\epsilon}{ \Theta 
        \left( \log^* \left(\frac{\sigma}{\epsilon} \right)  \right)}
    \quad \text{and} \quad
        \epsilon_i \coloneqq 
        \frac{\epsilon}{ \left( 2 \cdot i_{\text{max}}  \right) \cdot \left( |a_i|+ |b_i|\right)} =
                 \frac{\epsilon}{ \Theta 
        \left( \log^* \left(\frac{\sigma}{\epsilon} \right)  \cdot m_i \cdot \sigma \right)},
   \end{equation}
   and take
    \begin{equation}
       n_i =
       \begin{cases}
           \left\lceil \frac{1}{2 \epsilon_i^2} \log \left( \frac{2}{\delta_i} \right)\right\rceil  =
           O \left(\frac{\sigma^2}{\epsilon^2}
        \left( \log^*\left( \frac{\sigma}{\epsilon} \right)\right)^2
        \log\left( \frac{\log^*\left( \frac{\sigma}{\epsilon}\right)}{\delta}\right) 
        \right)
        &\text{ if } |i| \le 4
           \\ \\
           \left\lceil
        \left( \frac{2}{m_i^2} \frac{1}{\epsilon_i^2}  + \frac{2}{3} \frac{1}{\epsilon_i} \right) \cdot \log\left(\frac{2}{\delta_i} \right)
        \right\rceil = O \left( \left(  \frac{\sigma^2}{\epsilon^2} \cdot \log^*\left( \frac{\sigma}{\epsilon}\right)   + \frac{m_i \sigma}{\epsilon} \right) \cdot
        \log^*\left( \frac{\sigma}{\epsilon}\right) \cdot \log\left( \frac{ \log^*\left( \frac{\sigma}{\epsilon}\right) }{\delta}\right)\right) 
        &\text{ if } |i| \ge 5.
       \end{cases}
    \end{equation}
    Applying Hoeffding's inequality for each $|i| \le 4$ as in~\eqref{eq: start of Hoeffding} and Bernstein's inequality  for each $|i| \ge 5$ as in~\eqref{eq: start of Bernstein}--\eqref{eq: end of Bernstein} gives
    \begin{equation}
         \mathrm{Pr}\left( |p_{a_i} -  \hat{p}_{a_i}| > \epsilon_i \right)
         \le \delta_i
         \quad \text{and} \quad
         \mathrm{Pr}\left( |p_{b_i} -  \hat{p}_{b_i}| > \epsilon_i \right) \le \delta_i
    \end{equation}
    Summing up all $n_i$, we obtain
    \begin{equation}
    \begin{aligned}
         \sum_{i: |i| \le i_{\max}} n_i 
        &=O \left(  \left( i_{\max}  \frac{\sigma^2}{\epsilon^2} \cdot \log^* \left( \frac{\sigma}{\epsilon}\right)   + \sum_{i \le  i_{\max}} \frac{m_i \sigma}{\epsilon} \right) \cdot
        \log^*\left( \frac{\sigma}{\epsilon}\right) \cdot \log\left( \frac{ \log^*\left( \frac{\sigma}{\epsilon}\right) }{\delta}\right)\right)  \\
         &=
          O \left(  \left(\frac{\sigma^2}{\epsilon^2} \cdot  \left( \log^*\left( \frac{\sigma}{\epsilon} \right)\right)^2  
          +
          \frac{\sigma}{\epsilon}  \cdot  \frac{\sigma}{\epsilon} \right) \cdot
        \log^*\left( \frac{\sigma}{\epsilon}\right) \cdot \log\left( \frac{ \log\left( \frac{\sigma}{\epsilon}\right) }{\delta}\right)\right) \\
        &=
          O \left(  \frac{\sigma^2}{\epsilon^2} \cdot   \left( \log^*\left( \frac{\sigma}{\epsilon} \right)\right)^3   \cdot \log\left( \frac{ \log^*\left( \frac{\sigma}{\epsilon}\right) }{\delta}\right)\right),
    \end{aligned}
    \end{equation}
     where the second step follows from~\eqref{eq: subgaussian i_max}.

\section{Unknown Parameters}

\subsection{Proof of Theorem~\ref{thm: unknown eps} (Unknown Target Accuracy)}
\label{appendix: unknown target accuracy}
To establish that $\eps_T = O(\eps^*)$, we will compare the last round $T$ (see~\eqref{eq: last round T}) and $\tau^* \coloneqq \log_2 \left( \eps_0/\eps^* \right) = \log_2 \left( \sigma /2 \eps^* \right)$.
By the definition of $\tau^*$ and the definition of $n_{\text{ref}}$ (see~\eqref{eq: n_ref}), we have
\begin{equation}
\label{eq: simplified n_true}  \log \left( \sigma / \eps^* \right) = \Theta(\tau^*)
    \quad \text{and} \quad
     \frac{\sigma}{\eps^*} = \Theta\big(2^{\tau^*} \big)
     \implies
 n_{\text{ref}}(\eps^*, \delta, \sigma) 
  =  \Theta \left(  4^{(\tau^*)} \cdot  \left(  \tau^*\right)^3  \cdot \log\left( \frac{ \tau^* }{\delta}\right) \right).
\end{equation}
By using $ \eps_s = \sigma/2^{s} $, $\delta_{s} = \frac{6\delta}{\pi^2 s^2}$, and the fact that a sum of exponentially increasing terms is dominated by its last term, we have for any $\tau \ge 1$ that
\begin{equation*}
S(\tau) \coloneqq \sum_{s=1}^{\tau}  n_{\text{ref}}(\eps_s, \delta_s, \sigma)
=
\sum_{s=1}^{\tau}  \Theta \left(  4^s \cdot  s^3 \cdot \log\left( \frac{s}{\delta}\right) \right)
= 
\Theta \left(  4^{\tau} \cdot  \tau^3 \cdot \log\left( \frac{\tau}{\delta}\right) \right).
\end{equation*}
Using the definition of $T$ in \eqref{eq: last round T}, we have
\begin{equation}
\label{eq: n_true sandwiched}
    S(T) \le 
n_{\text{true}} -  n_{\text{loc}}(\delta, \lambda, \sigma)
< 
S(T+1)=
\Theta \left(  4^{T+1} \cdot  (T+1)^3 \cdot \log\left( \frac{T+1}{\delta}\right) \right),
\end{equation}
Combining~\eqref{eq: simplified n_true} and \eqref{eq: n_true sandwiched}, and recalling that $\epsilon^*$ is defined such that $n_{\text{true}} -  n_{\text{loc}}(\delta, \lambda, \sigma) = n_{\text{ref}}(\eps^*, \delta, \sigma)$,
we have
\[
 4^{(\tau^*)} \cdot  \left(  \tau^*\right)^3  \cdot \log\left( \frac{ \tau^* }{\delta}\right)
 =
 O \left( 4^{T+1} \cdot  (T+1)^3 \cdot \log\left( \frac{T+1}{\delta}\right) \right)
\]
which implies $T \ge  \tau^* - O(1)$. It follows that 
\[
    \eps_T 
    = \frac{\sigma}{2^{T}} 
    \le \frac{\sigma}{2^{\tau^* -O(1)}} 
    = 2^{O(1)} \cdot  \frac{\sigma}{2^{\tau^*}} 
    = O(\eps^*)
\]
as desired.

\subsection{Proof of Theorem~\ref{thm: partially unknown variance} (Adapting to Unknown Variance)}
\label{appendix: partially unknown variance}

Recall that our proposed method for this result was given in Section \ref{sec: unknown variance}.  
We first bound the sample complexity $n$. Recalling our choices of problem parameters in terms of $T$ (see~\eqref{eq: eps_i = r sigma_i/5}), we have
\begin{align}
    n &= 
    \sum_{i=0}^T n\left(\epsilon_i, \delta_i, \lambda, \sigma_i \right) 
    =
    \sum_{i=0}^T n\left(
     \frac{r \sigma_i}{5}, \frac{\delta}{T+1}, 
    \lambda,  
     \sigma_i
     \right) 
    =
    \sum_{i=0}^T  O \left(  \frac{1}{r^2} \cdot  \log^3\left( \frac{1}{r}\right)  \cdot \log\left( \frac{ T \log\left( \frac{1}{r}\right) }{\delta}\right) + \log\frac{\lambda}{\sigma_i}\right),
\end{align}
where the last step substitutes the sample complexity from Theorem \ref{thm: main}. 
Recalling that $T = \left\lceil \log_2 \left(\sigma_{\text{max}} / \sigma_{\text{min}} \right) \right\rceil$ and $\sigma_i = \sigma_{\min} \cdot 2^i $ (see~\eqref{eq: T = log(sigma_max/sigma_min)} and~\eqref{eq: eps_i = r sigma_i/5}), we have
\[
\sum_{i=0}^T \log_2 \frac{\lambda}{\sigma_i}
 = (T+1) \log_2 \frac{\lambda}{\sigma_{\min}} -
 \sum_{i=0}^T i
 = (T+1)\cdot \log_2 \frac{\lambda}{\sigma_{\min}} - \frac{T(T+1)}{2}
 =
 \Theta\left(T \log_2 \frac{\lambda}{\sqrt{\sigma_{\min} \sigma_{\max}} }
 \right).
\]
Combining the above two findings gives
\[
 n =
     O \left(  \log \left(\frac{\sigma_{\text{max}}}{\sigma_{\text{min}}} \right) \cdot 
     \frac{1}{r^2} \cdot  \log^3\left( \frac{1}{r}\right)  \cdot \log\left( \frac{  \log \left(\frac{\sigma_{\text{max}}}{\sigma_{\text{min}}} \right) \cdot  \log\left( \frac{1}{r}\right) }{\delta}\right) + 
      \log \left(\frac{\sigma_{\text{max}}}{\sigma_{\text{min}}} \right) \cdot  \log\frac{\lambda}{\sqrt{\sigma_{\min} \sigma_{\max}}}      
      \right) 
\]
as desired. 

We now show that the mean estimator is $(\eps, \delta)$-PAC, i.e.,
\begin{equation}
    \mathrm{Pr}\left(  |\hat{\mu}_{i^*} - \mu| \le  \eps \right) \ge 1- \delta.
\end{equation}
Let $k$ be the smallest index satisfying $\sigma_k \ge \sigma_{\text{true}}$:
\begin{equation}
\label{eq: sigma_k >= sigma_true}
k = \argmin_{i \ge 0} \{ \sigma_i \ge \sigma_{\text{true}} \}.
\end{equation}
For each $i \ge k$, the event 
\[
    \mathcal{E}_i = \{\mu \in I_i \}
    \quad \text{where } \quad
     I_i =  [ \hat{\mu}^{(i)} \pm \eps_i]
    \text{ is as defined as in } \eqref{eq: CI}
\]
occurs with probability at least
\[
\mathrm{Pr}\left(  \mathcal{E}_i \right) =
\mathrm{Pr}\left( \mu \in I_i \right) =
\mathrm{Pr}\left( |\hat{\mu}^{(i)} - \mu | \le \epsilon_i \right) \ge
1 - \delta_i
\]
by the subroutine's guarantee.
By the union bound, the ``good event'' $\mathcal{E} = \bigcap_{i \ge k} \mathcal{E}_i$ happens with probability at least
\[
\mathrm{Pr}\left(  \mathcal{E} \right) = 
\mathrm{Pr}\left(  \bigcap_{i \ge k} \mathcal{E}_i \right) =
1 -  \mathrm{Pr}\left(  \bigcup_{i \ge k} \neg \mathcal{E}_i \right)
\ge 1 - \sum_{i \ge k} \mathrm{Pr}\left( \neg \mathcal{E}_i \right) 
\ge 1- \sum_{i \ge k} \delta_i
\ge 1- \sum_{i =0}^T \delta_i
\ge 1- \delta.
\]
We now condition on event $\mathcal{E}$.
Observe that we have 
\begin{equation}
\label{eq: sigma_k < 2 sigma}
    \sigma_{k} =
    \begin{cases}
        \sigma_{\text{min}} & \text{if } k = 0 \\
         2 \sigma_{k-1} & \text{otherwise} 
    \end{cases}
    \implies 
    \sigma_{k} \le 2 \sigma_{\text{true}} = \frac{2 \eps}{r}
\end{equation}
due to~\eqref{eq: sigma_i}, the definition of $k$ (see~\eqref{eq: sigma_k >= sigma_true}) and the assumption that $\sigma_{\text{true}} \ge \sigma_{\text{min}}$. 
Based on~\eqref{eq: sigma_k < 2 sigma}, it is sufficient to show that 
\begin{equation}
\label{eq: |hat{mu}_{i^*} - mu| <= r sigma_k/2 is sufficient}
    |\hat{\mu}_{i^*} - \mu| \le \frac{r \sigma_k}{2}
\end{equation}
whenever the good event $\mathcal{E}$ holds. 

Recall that $i^*$ is the smallest index $i$ satisfying the feasibility condition in \eqref{eq: feasible condition}.  
Towards showing~\eqref{eq: |hat{mu}_{i^*} - mu| <= r sigma_k/2 is sufficient}, we first establish that $i^* \le k$, i.e., $\sigma_k$ is feasible.
Under event $\mathcal{E}$, we have $\mu \in I_k$ and also $\mu \in I_j$ for all $j > k$.
It follows that $k$ satisfies~\eqref{eq: feasible condition} and so $i^* \le k$ by definition.
If $i^* = k$, then
\[
    |\hat{\mu}_{i^*} - \mu| =
    |\hat{\mu}_{k} - \mu| \le
    \eps_k =
    \frac{r \sigma_k}{5} <
    \frac{r \sigma_k}{2}
\]
as desired.
On the other hand, if $i^* < k$, then by definition (see~\eqref{eq: sigma_i}), we have
\begin{equation}
\label{eq: sigma_k > 2 sigma_i*}
    \sigma_k \ge 2 \sigma_{i^*}.
\end{equation}
Furthermore, the two confidence intervals $I_{i^*}$ and $I_k$ must overlap by the feasibility of $i^*$ and the fact that $i^* < k$. Therefore, there is a common point $z$ such that $z \in I_{i^*}$ and $z \in I_k$. By the definition of the intervals (see~\eqref{eq: CI}), we have
$|\hat{\mu}_{i^*} - z| \le \epsilon_{i^*}$
and $|\hat{\mu}_{k} - z| \le \epsilon_{k}$,
which implies
\begin{equation}
\label{eq: mu_i^* almost as accurate}
    |\hat{\mu}_{i^*} - \hat{\mu}_{k}|
\le 
|\hat{\mu}_{i^*} - z| 
+ |z- \hat{\mu}_{k}| \le \epsilon_{i^*} + \epsilon_{k}
\end{equation}
by the triangle inequality.
Using the triangle inequality a second time along with~\eqref{eq: mu_i^* almost as accurate}, event $\mathcal{E}$, the choice of $\eps_i$ in~\eqref{eq: eps_i = r sigma_i/5}, and~\eqref{eq: sigma_k > 2 sigma_i*}, we have
\[
    |\hat{\mu}_{i^*} - \mu| \le 
    |\hat{\mu}_{i^*} - \hat{\mu}_{k}| 
    +
    | \hat{\mu}_{k} - \mu| 
    \le
    \eps_{i^*} + 2\eps_{k} =
    \frac{r \sigma_{i^*}}{5}  + \frac{2r \sigma_k}{5} \le
    \frac{r \sigma_{k}}{2},
\]
thus giving the desired sufficient condition~\eqref{eq: |hat{mu}_{i^*} - mu| <= r sigma_k/2 is sufficient}.
\section{Details of Two-stage Mean Estimator}
\label{appendix: gray code}

Here we provide the technical details for the non-adaptive localization protocol described in Section~\ref{sec: two-stage}. 
The goal of this localization protocol is to identify an interval $I$ of length  $O(\sigma)$ that contains the mean $\mu$ with high probability. 
The core idea is adapted from~\cite{cai2024distributed}, whose focus is on Gaussian distributions.
We modify their approach to handle our general non-parametric family $\mathcal{D}(\lambda, \sigma)$.
The protocol works by encoding the  location of the mean
using a binary Gray code of length $K = \Theta(\log(\lambda/
\sigma))$, and estimating each of these $K$ bits by aggregating responses from suitably chosen non-adaptive queries. We now formalize the necessary definitions and describe the procedure.

\begin{definition}[Gray function]
     For integers $k \ge 0$, we let
     $g_k \colon [0, 1] \to \{0, 1\}$ be the $k$-th Gray function, defined by
  \[
    g_k(x) \coloneqq
   \begin{cases}
    0  & \text{if } \left\lfloor 2^k \cdot x\right\rfloor \bmod 4 \in \{0, 3\} \\
    1  & \text{if }  \left\lfloor 2^k \cdot x \right\rfloor \bmod 4 \in \{1,2\}
    \end{cases}.
    \]
\end{definition}

\begin{definition}[Change points set] \label{def:Gk}
    The set $G_k$ of change points for $g_k$ is defined as
    the collection of points $x \in [0, 1]$ where $g_k(x)$ changes its value from 0 to 1 or from 1 to 0. Formally, we define 
    \[
    G_k
    =
    \left\{(2j-1) \cdot 2^{-k}:1\le j\le2^{k-1}\right\}
     =
    \left\{ x \in [0,1] :\lim_{y\to x^-}g_k(y)\neq\lim_{y\to x^+}g_k(y)\right\}.
    \]
    Note that the $G_k$ are pairwise disjoint, i.e., 
    $G_k \cap G_{k'} = \varnothing$ for $k \ne k'$.
\end{definition}

\begin{definition}[Decoding]
    For any $K \ge 1$, we let $\mathrm{Dec}_K\colon \{0, 1\}^K \to 2^{[0,1]}$ be the decoding function defined by
    \[
    \mathrm{Dec}_K(y_1,\dots,y_K) \coloneqq
    \left\{
        x  \in [0,1]  : g_k(x) = y_k \quad \text{for } 1\le k\le K \right\}
    \]
    This is a dyadic interval of length $2^{-K}$ that is consistent with the gray code bits $y_1, y_2, \dotsc, y_K$, and so we can express it as follows for some $x_0 \in [0,1]$:
    \[
    \mathrm{Dec}_K(y_1,\dots,y_K) =
    \left[ x_0,x_0+2^{-K} \right]\subset[0,1].
    \]
\end{definition}
With these definitions in mind, we now describe the localization procedure.
\begin{enumerate}
    \item 
    We first rescale
    \[
      X'_i = \frac{X_i + \lambda}{2\lambda} \in [0, 1]
      \quad \text{and} \quad
      \mu' = \frac{\mu + \lambda}{2\lambda} \in [0, 1],
    \]
    and note that the resulting variance scales as follows:
    \begin{equation}
        \EE \big[ |X'_i-\mu'|^2 \big]
        \le
        \left(\frac{\sigma}{2\lambda} \right)^{2}. \label{eq:gray_variance}
    \end{equation} 

    \item 
    We view the samples as being collected in groups.  
    Let the number of groups to be
    \begin{equation}
    \label{eq: number of group}
        K = \left\lfloor
    \log_2 \left(\frac{2\lambda}{\sigma}\right) - 3
    \right\rfloor ,
    \end{equation}
    with each group having the fixed number of samples
    $
    J = \left\lceil 8\,\log\frac{3K}{\delta}\right\rceil.
    $
    Thus, the total number of samples used (for localization) is
    \[
    KJ = 
    \Theta\left( \log \left( \frac{\lambda}{\sigma}\right) \cdot \log \frac{\log(\lambda/\sigma)}{\delta}\right).
    \]

   \item
    For sample $j$ in group $k$, the agent sends the single bit
    \[
    Z_{k,j}=g_k(X_{k,j}'),
    \]
    where $X_{k,j}'$ is the unquantized transform sample.

    \item

    For each group $k = 1, \dotsc, K$, the learner computes the majority bit
    \[
    \hat z_k = \mathrm{Maj}\{Z_{k,1}, \dotsc, Z_{k,J}\}
    =
    \begin{cases}
    1 & \text{if } \sum_j Z_{k,j} \ge J/2,\\
    0 & \text{otherwise.}
\end{cases}
    \]

    \item 
    The learner first computes the interval $\left[ x_0,x_0+2^{-K} \right] = \mathrm{Dec}_K(\hat z_1,\dotsc, \hat z_K)$, and then widens it by 
    shifting the left end and right end by $2^{-(K+2)}$:
    \begin{equation}
        \label{widened interval}
         I' = \left[x_0 - 2^{-(K+2)},\; x_0 + 2^{-K} + 2^{-(K+2)} \right]\cap[0,1].
    \end{equation}
    Finally, it scales and shifts the interval $I' = [L', U']$
    by using the transformation
    \[
       I = 2 \lambda I' - \lambda =
       \left[2\lambda L' -\lambda,\ 2\lambda  U'-\lambda\right]
    \]
    and returns this as the final interval. Note that the length satisfies
    \begin{equation}
    \label{eq: alternative localization interval length}
         |I| =  
    2 \lambda \cdot (U' - L') 
    \le
    2\lambda \cdot \left(2^{-K} + 2 \cdot 2^{-(K+2)} \right) 
    = 2^{-K} \cdot 3 \lambda 
    = O(\sigma),
    \end{equation}
    where the last step follows from the choice of $K$ in \eqref{eq: number of group}.
\end{enumerate}
Before proving Theorem~\ref{thm: alternative localization}, we first state three useful lemmas below. Lemma~\ref{eq: Lemma 17 restated} is a restatement of~\cite{cai2024distributed}[Lemma 17] (whose proof is elementary and straightforward), while the other two lemmas bound the encoding and decoding error probability.
\begin{lemma}
\label{eq: Lemma 17 restated}
{(\cite{cai2024distributed}[Lemma 17])}
Let $I'$ be the widened interval as stated in~\eqref{widened interval}.
If each $k \in \{1, \dotsc, K$\} satisfies  the condition
\begin{equation}
\label{eq: Lemma 17 condition}
    \underbrace{\inf_{y\in G_k} |\mu'-y|}_{=: d_k} < 2^{-(K+2)}
\quad\text{or}\quad
\hat z_k= g_k(\mu'),
\end{equation}
then it holds that $\mu' \in I'$. Note that there is at most one $k$ satisfying the condition $d_k < 2^{-(K+2)}$.
\end{lemma}

\begin{lemma}
\label{lem: g_k(X'_{k,j}) neq g_k(mu')}
For each $k  = 1, \dotsc K$ and each $j = 1, \dotsc, J$, we have
\[
\mathrm{Pr} \left( g_k(X'_{k,j})\neq g_k(\mu') \right)
\le  \left(\frac{ \sigma}{2 \lambda d_k} \right)^2,
\]
where $d_k = \inf_{y\in G_k} |\mu'-y|$ is the distance from the transformed mean to the set~$G_k$ from Definition \ref{def:Gk}.
\end{lemma}

\begin{proof}
We first claim that
\begin{equation}
\label{eq: intermediate g_k not matching}
    \mathrm{Pr} \left( g_k(X'_{k,j})\neq g_k(\mu') \right)
    \le 
    \mathrm{Pr}\left(|X'_{k,j}-\mu'| \ge d_k \right).
\end{equation}
Before proving this, we note that given that it holds, Chebyshev's inequality (with the variance bound in \eqref{eq:gray_variance}) gives the desired bound:
\[
\mathrm{Pr} \left( g_k(X'_{k,j})\neq g_k(\mu') \right)
\le 
\mathrm{Pr}\left(|X'_{k,j}-\mu'| \ge d_k \right)
\le
\left(\frac{ \sigma}{2\lambda d_k} \right)^2.
\]
It remains to establish~\eqref{eq: intermediate g_k not matching}, or equivalently
 \begin{equation}
 \label{eq: event inclusion}
     \mathrm{Pr}\left(|X'_{k,j}-\mu'| < d_k\right)
\le \mathrm{Pr} \left( g_k(X'_{k,j}) = g_k(\mu') \right).
 \end{equation}
This follows from the event implication
\[
\left\{ |X'_{k,j}-\mu'| < d_k \right\}
\implies
\left\{g_k(X'_{k,j}) =  g_k(\mu') \right\},
\]
which follows immediately from the definition of $d_k$. 
\end{proof}

\begin{lemma}[Majority‐vote reliability]
\label{eq: majority vote}
Fix a group $k \in \{1, \dotsc, K\}$. Suppose that each i.i.d. sample $X'_{k,j}$ with $j \in \{1, \dotsc, J\}$ satisfies
\[
\mathrm{Pr} \left( g_k(X'_{k,j})\neq g_k(\mu') \right) \le \frac{1}{4}.
\]
Then, under the choice $J = \lceil 8 \log\frac{3K}{\delta} \rceil$, the majority vote $\hat z_k = \mathrm{Maj}\{g_k(X'_{k,1}, \dotsc, g_k(X'_{k,J}) \} $ satisfies
\[
\mathrm{Pr}\left(\hat z_k \neq  g_k(\mu')\right)
\le
\exp(-J/8)
\le
\frac{\delta}{3K}.
\]
\end{lemma}
\begin{proof}
Let $ B_j := \mathbf{1}\{ g_k(X'_{k,J})  \neq g_j(\mu')\} $, which gives $ B_j \sim \mathrm{Bern}(p_j)$ with $p_j \le 1/4$. 
Let $ S = \sum_{j=1}^{J} B_j $ count the number of errors in the group.
The majority vote is incorrect only when at least half are wrong:
\[
 \mathrm{Pr}\left(\hat z_k \neq  g_k(\mu')\right) = \mathrm{Pr}\left( S \ge \frac{J}{2} \right) = \mathrm{Pr}\left( S - \mathbb{E}[S] \ge \frac{J}{2} - \mathbb{E}[S]\right).
\]
Since $\mathbb{E}[S]  \le J/4 $, applying Hoeffding inequality yields
\[
\mathrm{Pr}\left( S - \mathbb{E}[S] \ge \frac{J}{2} - \mathbb{E}[S]\right)
\le
\mathrm{Pr}\left( S - \mathbb{E}[S] \ge \frac{J}{4} \right)
\le \exp\left( -\frac{J}{8} \right)
\le
\frac{\delta}{3K}
\]
as desired.
\end{proof}

\begin{proof}[Proof of Theorem~\ref{thm: alternative localization}]
Given~\eqref{eq: alternative localization interval length}, it remains to show with probability at least $1-\delta/2$ that $\mu \in I$, or equivalently, the scaled mean $\mu' = (\mu + \lambda)/(2 \lambda)$ lies in $I'$.
In view of Lemma~\ref{eq: Lemma 17 restated}, we define the ``good events''
\[
E_k = \left\{ d_k < 2^{-(K+2)} \text{ or } 
\hat z_k= g_k(\mu') \right\}
\]
and show that
\[
\mathrm{Pr} \left( \bigcup_{k=1}^K E_k \right) \ge 1 - \frac{\delta}{2}.
\]
By the union bound, it is sufficient to show that each ``bad event'' $\bar{E}_k$ happens with probability at most
\[
\mathrm{Pr}(\bar{E}_k) = 
\mathrm{Pr} \left( d_k \ge 2^{-(K+2)} \text{ and } 
\hat z_k \ne g_k(\mu') \right) \le  \frac{\delta}{2K}.
\]
Fix an arbitrary $k \in \{1, \dots, K \}$.
If $d_k < 2^{-(K+2)}$, then $\mathrm{Pr}(\bar{E}_k) = 0$. Therefore, we may assume without loss of generality that 
\[
    d_k \ge 2^{-(K+2)}
    \iff 
    \frac{1}{d_k} \le 2^{K+2} = 4 \cdot 2^K.
\]
Using this assumption,  the choice of $K$ (see~\eqref{eq: number of group}), and Lemma~\ref{lem: g_k(X'_{k,j}) neq g_k(mu')}, we have 
\[
\mathrm{Pr} \left( g_k(X'_{k,j})\neq g_k(\mu') \right)
\le 
\left(\frac{ \sigma}{2\lambda d_k} \right)^2
\le
\left(\frac{ \sigma}{2\lambda} \cdot 4 \cdot \frac{2\lambda}{\sigma}\cdot \frac{1}{8} \right)^2
= \frac{1}{4}.
\]
It then follows from Lemma~\ref{eq: majority vote} that 
\[
    \mathrm{Pr}\left(\hat z_k \neq  g_k(\mu')\right)
\le \frac{\delta}{3K}
\implies
\mathrm{Pr}(\bar{E}_k)
\le  \frac{\delta}{3K} <  \frac{\delta}{2K}
\]
as desired.
\end{proof}


\end{document}